\documentclass[lettersize,journal]{IEEEtran}
\usepackage{amsmath,amsfonts}
\usepackage{algorithmic}
\usepackage{algorithm}
\usepackage{array}
\usepackage[caption=false,font=normalsize,labelfont=sf,textfont=sf]{subfig}
\usepackage{textcomp}
\usepackage{stfloats}
\usepackage{url}
\usepackage{verbatim}
\usepackage{graphicx}
\usepackage{cite}
\usepackage{pgfplots} 
\pgfplotsset{compat=1.16}
\usepackage{subcaption}
\usepackage{amsmath}
\usepgfplotslibrary{groupplots}

\usepackage{caption}
\usepackage{amsthm}
\usepackage{enumitem}
\usepackage{subcaption}
\usepackage{tikz}
\usepackage{array}
\usepackage{balance}
\usepackage{bigstrut,array,multirow,tabularx}
\usepackage{caption}
\usepackage{amsmath}
\usepackage{amssymb}
\usepackage{pifont}
\newcommand{\cmark}{\ding{51}}%
\newcommand{\xmark}{\ding{55}}%
\usepackage{dsfont}
\usepackage{mathtools}

\usepackage{algorithm}
\usepackage{algorithmic}
\usepackage{makecell} 
\usepackage{colortbl}
\usepackage{setspace}

\usepackage{booktabs}
\theoremstyle{remark}
\newtheorem*{remark}{Remark}

\definecolor{lightblue}{RGB}{0.93,0.95,1.0} 
\definecolor{lightred}{RGB}{1.0,0.93,0.93} 
\newtheorem{definition}{Definition}[section]
\newtheorem{theorem}{Theorem}[section]

\definecolor{bblue}{HTML}{FFE333}
\definecolor{rred}{HTML}{7FFF00}
\definecolor{ggreen}{HTML}{87CEEB}

\begin{document}

\title{Memory Is No Longer a Bottleneck: Memory-Efficient Graph Filtering for Scalable Collaborative Filtering}

\author{Jin-Duk~Park
        and~Won-Yong~Shin,~{\em Senior Member}, {\em IEEE}

        \IEEEcompsocitemizethanks{\IEEEcompsocthanksitem J.-D. Park and W.-Y. Shin are with the School of Mathematics and Computing (Computational Science and Engineering), Yonsei University, Seoul 03722, Republic of Korea.
 \protect\\
E-mail: \{jindeok6, wy.shin\}@yonsei.ac.kr
(Corresponding author: Won-Yong Shin.)}
}

\markboth{Submitted to IEEE Transactions on Knowledge and Data Engineering}%
{Shell \MakeLowercase{\textit{et al.}}: A Sample Article Using IEEEtran.cls for IEEE Journals}


\maketitle

\begin{abstract}
Graph convolutional networks (GCNs) have demonstrated significant success in capturing complex user--item relationships for collaborative filtering (CF). However, due to their reliance on extensive model training, training-free graph filtering (GF)-based CF methods have emerged as a promising alternative, offering computational efficiency by smoothing graph signals via matrix operations. In particular, polynomial GF-based approaches demonstrate improved accuracy through their ability to design more expressive and flexible filtering functions. Despite these advantages, existing GF methods suffer from a critical memory bottleneck: they necessitate storing the full item similarity graph, incurring prohibitive memory costs for large-scale datasets, which limits their practical applicability. To tackle this challenge, we propose \textsf{Mem-GF} (\underline{Mem}ory-efficient \underline{GF}), a new GF-based CF method that departs from conventional designs by principally leveraging the structure of \textit{Krylov subspaces} as a core mechanism for approximating polynomial graph filters without explicitly storing the item similarity graph. We theoretically analyze the minimum Krylov subspace size that guarantees lossless approximation. Through extensive experiments, we demonstrate that \textsf{Mem-GF} achieves up to 5.74$\times$ lower memory usage and 4.38$\times$ speedup in runtime, while consistently exceeding the recommendation accuracy of state-of-the-art GF and GCN-based methods. \textsf{Mem-GF} robustly scales to datasets with tens of millions of interactions, establishing itself as a practically viable and theoretically grounded solution for efficient CF. The source code of \textsf{Mem-GF} is available at \url{https://github.com/jindeok/Mem-GF}.
\end{abstract}


\section{Introduction}

\subsection{Background and Motivation} 
\IEEEPARstart{R}{ecommender} systems have become a cornerstone component of modern online platforms \cite{linden2003amazon,covington2016deep,gomez2015netflix,gong2020edgerec}. By leveraging the information of historical user--item interactions, collaborative filtering (CF) has proven one of the most effective recommendation techniques in uncovering latent user interests and item affinities, facilitating personalization and relevance for positive user experiences \cite{he2017neural,liang2018variational,fu2018novel,he2020lightgcn}. 

On one hand, graph convolutional networks (GCNs) \cite{DBLP:conf/iclr/KipfW17, he2020lightgcn, wang2019neural} have emerged as a state-of-the-art CF solution. Thanks to their intrinsic ability to model the high-order connectivity between users and items, GCN-based approaches have achieved high accuracy in various recommendation scenarios \cite{he2020lightgcn, wang2019neural,park2023criteria, wu2019session, kim2025leveraging}. However, GCNs inherently suffer from scalability issues due to their recursive message-passing nature, which may limit their applicability to large-scale systems \cite{borisyuk2024lignn, xia2023graph}. As a result, the training time required for GCNs is often significantly longer compared to other architectures such as multi-layer perceptrons (MLPs) that are widely used for CF \cite{xia2023graph, zhang2021graph}. This increased computational demand can hinder the efficiency of real-time recommendation systems, making it crucial to explore alternative methods that balance accuracy and scalability.

On the other hand, from the spectral viewpoint, GCNs can be regarded as approximate and learnable graph filters via parameter optimization \cite{DBLP:conf/iclr/KipfW17, shen2021powerful, liu2023personalized}. To mitigate the extensive training time associated with GCNs, recent research on CF has increasingly focused on graph filtering (GF)-based CF methods that leverage \textit{predefined graph filters} instead of learnable ones \cite{shen2021powerful, park2024turbo,xia2024hierarchical,xia2022fire, liu2023personalized}. Due to their training-free nature, GF has been successfully applied beyond standard training-based CF to diverse recommendation scenarios, including sequential \cite{xia2022fire}, group \cite{kim2025leveraging}, cross-domain \cite{lee2024graph}, multi-modal \cite{Roh2026MM-GF}, and multi-criteria \cite{park2025criteria} recommendations, demonstrating its broad applicability and effectiveness. By using the item similarity graph, these GF-based CF methods transform each user's interaction vector (\textit{i.e.}, graph signals) via various forms of graph filters, such as linear/ideal low-pass filters (LPFs) or polynomial approximations, to improve computational efficiency of the GF process~\cite{shen2021powerful,xia2022fire,liu2023personalized,park2024turbo,xia2024hierarchical,qin2025polycf}.
\begin{figure}[t]
    \centering
    \includegraphics[width=0.99\columnwidth]{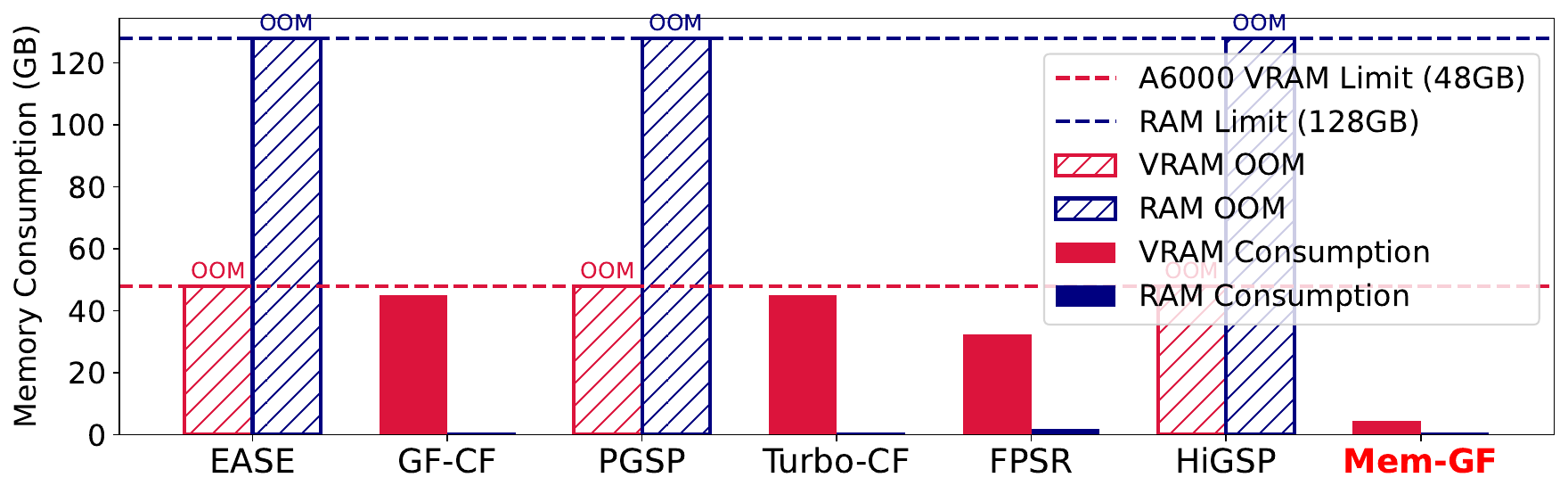}
    \vspace{-3mm}
    \caption{Memory usage (RAM/VRAM) comparison during preprocessing among six representative GF-based CF methods and our method (\textsf{Mem-GF}) on the Amazon-book dataset. Here, `OOM' represents an out-of-memory error.}
    \label{fig;motive}
\end{figure}
\begin{figure}[t]
    \centering
    \includegraphics[width=0.95\columnwidth]{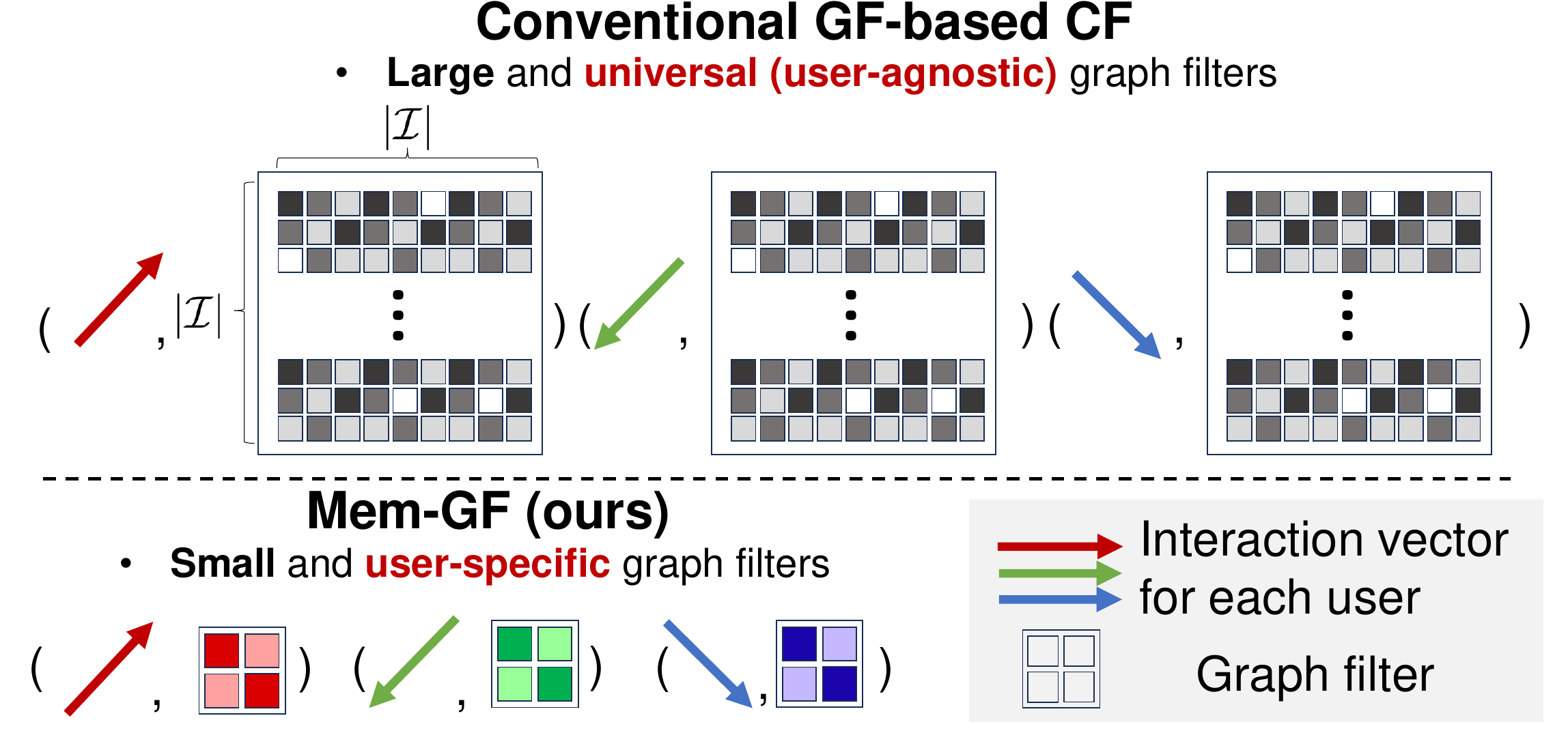}
    \caption{Comparison of space consumption between conventional GF-based CF methods and the proposed \textsf{Mem-GF}. Here, $|\mathcal{I}|$ represents the number of items in a dataset.}
    \label{fig;intro}
\end{figure}

Despite these merits, many GF-based CF methods \cite{shen2021powerful,xia2022fire,liu2023personalized,park2024turbo,xia2024hierarchical} still face a critical {\it memory bottleneck} when high-accuracy filters are deployed at scale. Fig.~\ref{fig;motive} showcases the RAM and VRAM ({\it i.e.}, GPU memory) usage of six representative GF-based approaches (such as EASE \cite{steck2019embarrassingly}, GF-CF \cite{shen2021powerful}, PGSP \cite{liu2023personalized}, Turbo-CF \cite{park2024turbo}, FPSR \cite{wei2023fine}, and HiGSP \cite{xia2024hierarchical}) on the Amazon-book dataset \cite{he2016ups}. In a single-node environment equipped with 48\,GB of VRAM and 128\,GB of RAM, most of these methods suffer from out-of-memory (OOM) errors, underscoring their limited feasibility for large-scale recommendation scenarios. The core design issue is not merely implementation, but the reliance on a {\it universal} ({\it i.e.}, user-agnostic) graph filter whose dimension is $|\mathcal I|\times|\mathcal I|$, where $|\mathcal{I}|$ is the item cardinality, {\it i.e.}, the number of items. Even when sparse matrix operations, top-$k$ eigenvalue components, or polynomial graph filters reduce computation, the filtering stage still tends to depend on the full-size item similarity matrix. As illustrated in Fig.~\ref{fig;intro}, applying such filters uniformly to all user signals severely limits scalability and can make real-time inference infeasible under memory-constrained environments.

\subsection{Main Contribution}
To address this challenge, we introduce \textsf{Mem-GF}, a memory-efficient GF method for CF that goes beyond straightforward Krylov subspace usage. Instead of treating Krylov methods \cite{saad2011numerical,trefethen1997numerical} as a computational shortcut, \textsf{Mem-GF} uses them as the filtering principle itself. For each user, a Krylov subspace captures progressively transformed directions of the user's interaction vector under repeated applications of the item similarity matrix. This enables polynomial graph filtering without explicitly storing the full item similarity matrix. Importantly, \textsf{Mem-GF} facilitates the design of versatile graph filters through high-order polynomial approximations, exemplified by the fifth-order polynomial that closely approximates Gaussian filters widely used in image and graph domains \cite{ito2000gaussian, li2022g}. Notably, \textsf{Mem-GF} offers significant advantages in both preprocessing and inference efficiency. In particular, its memory-efficient and low-latency {\it inference} capability addresses a critical requirement for deploying GF-based recommenders in real-world environments, where tight hardware constraints and real-time responsiveness are essential \cite{jiang2021microrec,lian2020lightrec}. 

Our contributions are summarized as follows:
\begin{itemize}
    \item \textbf{Problem identification}: We clarify the universal ({\it i.e.}, user-agnostic) graph filter design that causes the memory bottleneck in many GF-based CF methods and empirically demonstrate its impracticality on large-scale datasets.
    \item \textbf{Principled Krylov-based solution}: We propose a novel GF approach that integrates Krylov subspaces at its core, achieving per-user filter approximation without storing the full item similarity matrix, thereby enhancing filter design flexibility.
    \item \textbf{Theoretical support}: We rigorously analyze the condition on the number of Krylov subspace bases for which no approximation error occurs. We also theoretically show that \textsf{Mem-GF} achieves {\it linear space/time complexity} to the number of items, users, and interactions.
    \item \textbf{Comprehensive evaluation}: We empirically show that \textsf{Mem-GF} delivers \textbf{5.74$\times$ memory savings} and \textbf{4.38$\times$ speedup}. Despite being primarily designed for improving efficiency, it still achieves state-of-the-art recommendation accuracy across all benchmark datasets.
\end{itemize}

\subsection{Organization and Notations}
The remainder of this paper is organized as follows. Section~\ref{app:rel} surveys related work. Section~\ref{section 2} introduces the preliminaries by reviewing GF and Krylov subspaces. Section~\ref{section 3} presents the proposed \textsf{Mem-GF} method, including its methodological details, theoretical analysis, and complexity analysis. Section~\ref{section 4} reports extensive experimental results and analyses evaluating efficiency, scalability, and accuracy. Section~\ref{section 5} concludes the paper with a discussion of future research directions.

Table~\ref{tab:notation} summarizes the notation that is used in this paper. This notation will be formally defined in the following sections when we introduce our methodology.

\begin{table}[t]
\scriptsize
\centering
\caption{Summary of notations.}
\label{tab:notation}
\begin{tabular}{@{}p{0.13\columnwidth}p{0.73\columnwidth}@{}}
\toprule
\textbf{Notation} & \textbf{Description} \\
\midrule
$\mathcal{U},\ \mathcal{I}$ & Set of users, $\mathcal{U}$, and set of items, $\mathcal{I}$\\
$\mathbf{R}$ & User--item interaction matrix\\
$\mathbf{r}_u$ & $u$-th row of the interaction matrix $\mathbf{R}$ (graph signal for user $u$) \\
$\tilde{\mathbf{R}}$ & Degree-normalized interaction matrix: $\tilde{\mathbf{R}} = \mathbf{D}_\mathcal{U}^{\gamma} \mathbf{R} \mathbf{D}_\mathcal{I}^{1-\gamma}$\\
$\bar{\mathbf{R}}$ & Adjusted interaction matrix: $\bar{\mathbf{R}}=\tilde{\mathbf{R}}^{\circ s}$\\
$\mathbf{D}_\mathcal{U},\ \mathbf{D}_\mathcal{I}$ & Degree matrices: $\mathbf{D}_\mathcal{U}=\mathrm{diag}(\mathbf{R}\mathbf{1}),\ \mathbf{D}_\mathcal{I}=\mathrm{diag}(\mathbf{1}^\top \mathbf{R})$\\
$\tilde{\mathbf{P}}$ & Item similarity matrix: $\tilde{\mathbf{P}}=\bar{\mathbf{R}}^\top \bar{\mathbf{R}}$ for \textsf{Mem-GF} and $\tilde{\mathbf{P}}=\tilde{\mathbf{R}}^\top \tilde{\mathbf{R}}$ for the other methods\\
$\mathbf{L},\ \mathbf{U},\ \boldsymbol{\Lambda}$ & Graph Laplacian $\mathbf{L}$, its eigenvectors $\mathbf{U}$, and eigenvalues $\boldsymbol{\Lambda}$\\
$H(\cdot),\ h(\cdot)$ & Graph filter $H(\cdot)$ and its frequency response  $h(\cdot)$\\
$\mathbf{s}_u$ & Predicted scores for user $u$: $\mathbf{s}_u=\|\mathbf{r}_u\|_2\,\mathbf{Q}_u\,f(\mathbf{T}_u)\,\mathbf{e}_1$\\
$\mathcal{K}_K(\cdot)$ & $K$-dimensional Krylov subspace\\
$\mathbf{Q}_u,\ \mathbf{T}_u$ & User-specific Krylov basis $\mathbf Q_u$ and projected tridiagonal matrix $\mathbf T_u=\mathbf Q_u^\top\tilde{\mathbf P}\mathbf Q_u$ in the Krylov subspace\\
$K$ & Size of Krylov basis\\
$\gamma$& Hyperparameter for normalizing $\mathbf{R}$\\
$s$ & Hyperparameter for adjusting $\tilde{\mathbf{R}}$ via the Hadamard power\\
$\mathbf{e}_1$ & First canonical basis vector in $\mathbb{R}^{K}$\\
\bottomrule
\end{tabular}
\end{table}

\section{Related Work}
\label{app:rel}
In this section, we review CF methods related to \textsf{Mem-GF}, with emphasis on scalable GF-based CF and Krylov-based graph learning.

\textbf{MF-based CF methods.} MF remains a foundational technique in CF, representing users and items as latent vectors and using their dot product to approximate interactions. Early approaches relied on low-rank decompositions~\cite{he2016vbpr}, while NeuMF~\cite{he2017neural} enhanced expressiveness by introducing an MLP-based similarity function. GRMF~\cite{rao2015collaborative} integrated graph Laplacian regularization, and DMF~\cite{xue2017deep} leveraged deep neural architectures. HOP-rec~\cite{yang2018hop} bridged MF and GCN by explicitly capturing higher-order relationships in user--item graphs.

\textbf{GCN-based CF methods.} Viewing user--item interactions as a bipartite graph enables GCN-based models to exploit high-order connectivity. GC-MC~\cite{berg2017graph} merged GCNs with matrix completion for link prediction. PinSage~\cite{ying2018graph} scaled graph convolution with random walks. SpectralCF~\cite{zheng2018spectral} used spectral graph structure, NGCF~\cite{wang2019neural} modeled higher-order collaborative signals, LightGCN~\cite{he2020lightgcn} simplified propagation layers, and SGL~\cite{wu2021self} adopted contrastive learning for robust representation learning.

\textbf{Autoencoder-based CF methods.} Autoencoder-based approaches reconstruct each user's interaction vector in a single latent space. CDAE~\cite{wu2016collaborative} incorporated denoising, Multi-DAE and Multi-VAE~\cite{liang2018variational} improved reconstruction with deep and variational encoders, RecVAE~\cite{shenbin2020recvae} added regularization, and MD-CVAE~\cite{zhu2022mutually} injected item embeddings into latent user modeling.

\textbf{GF-based CF methods.} Training-free GF-based CF methods avoid iterative model training by filtering user signals on an item similarity graph. GF-CF~\cite{shen2021powerful} uses linear and ideal LPFs; EASE~\cite{steck2019embarrassingly} can be interpreted from a graph signal processing viewpoint; PGSP~\cite{liu2023personalized} augments the similarity graph; HiGSP~\cite{xia2024hierarchical} customizes filters by user clusters; Turbo-CF~\cite{park2024turbo} accelerates inference with predefined polynomial LPFs; SGFCF~\cite{peng2024powerful} uses top-$k$ singular values; and ChebyCF~\cite{kim2025graph} employs Chebyshev polynomial filtering. These methods improve scalability through sparse matrix operations, top-$k$ eigenvalue components, polynomial graph filters, or user clustering, but they still commonly rely on a large item similarity graph or a full-size graph filter. In contrast, \textsf{Mem-GF} performs GF in a user-specific Krylov subspace and avoids constructing the full item similarity matrix.

\textbf{Krylov-based graph learning.} LanczosNet~\cite{liao2019lanczosnet} also uses Krylov subspaces, but for learnable spectral graph convolution in GNNs. Its goal is representation learning over graph nodes, requiring training and inference over the graph. \textsf{Mem-GF} differs in its purpose and inference setting: it is a training-free CF method that builds per-user Krylov subspaces for the item similarity matrix to reduce memory usage during recommendation. Thus, Krylov subspaces are not used for training-based GNNs, but as the core mechanism for memory-efficient graph filtering in CF.
\section{Preliminaries} 
This section presents the preliminaries, providing an overview of the fundamental concepts underlying GF and Krylov subspaces.
\label{section 2}
\subsection{Characterization of GF}
 We begin by presenting the basic concepts of GF (also known as graph signal processing). First, let us consider a weighted, undirected graph $G = (V, E)$, represented by an adjacency matrix $\mathbf{A} \in\mathbb{R}^{|V|\times |V|}$. A graph signal assigns a scalar or feature vector to each node; in this paper, we primarily consider a scalar signal represented as $\mathbf{x}=[x_1, x_2, \ldots, x_{|V|}]^T\in\mathbb{R}^{|V|}$, where $x_i$ indicates the signal strength of node $i$. The smoothness of $\mathbf{x}$ on $G$ is quantified by the graph quadratic form, a measure based on the graph Laplacian $\mathbf{L} = \mathbf{D} - \mathbf{A}$,\footnote{Alternatively, one may use the normalized Laplacian $\mathbf{L} = \mathbf{I} - \tilde{\mathbf{A}}$, where $\tilde{\mathbf{A}}=\mathbf{D}^{-1/2}\mathbf{A}\mathbf{D}^{-1/2}$.} where $\mathbf{D}=\text{diag}(\mathbf{A}\mathbf{1})$ is the degree matrix and $\mathbf{1}$ denotes the all-ones vector of any dimension for simplicity. The smoothness measure $S(\mathbf{x})$ is formally given by 
\begin{equation}
    S(\mathbf{x}) = \sum_{i,j}A_{i,j}(x_i-x_j)^2 =  \mathbf{x}^T \mathbf{L} \mathbf{x}.
\end{equation}
The smaller the value of $S(\mathbf{x})$, the smoother the signal $\mathbf{x}$ on the graph \cite{shuman2013emerging, shen2021powerful}. 
Next, we define the graph Fourier transform (GFT) of a signal $\mathbf{x}$ as $\hat{\mathbf{x}} \;=\; \mathbf{U}^T \,\mathbf{x}$, where $\mathbf{U} \in \mathbb{R}^{|V|\times |V|}$ comprises the eigenvectors of $\mathbf{L}$. Specifically, from the eigen-decomposition $\mathbf{L} \;=\; \mathbf{U} \,\boldsymbol{\Lambda}\,\mathbf{U}^T,$ let $\boldsymbol{\Lambda} = \mathrm{diag}(\lambda_1,\dots,\lambda_{|V|})$ with ordered eigenvalues $\lambda_1 \le \dots \le \lambda_{|V|}$. Because the GFT is an orthogonal linear transform, the inverse GFT is 
$\mathbf{x} \;=\; \mathbf{U}\,\hat{\mathbf{x}}.$ Finally, we formalize the notions of a graph filter and graph convolution:
\vspace{-1mm}
\begin{definition}[Graph filter \cite{shuman2013emerging,shen2021powerful,xia2022fire,ortega2018graph}]
Given a graph Laplacian $\mathbf{L}$, a graph filter $H(\mathbf{L})\in\mathbb{R}^{|V|\times|V|}$ is defined as
\begin{equation}
    H(\mathbf{L}) \;=\; \mathbf{U} \,\mathrm{diag}\bigl(h(\lambda_1), \dots, h(\lambda_{|V|})\bigr)\,\mathbf{U}^T,
\end{equation}
where $h : \mathbb{R}_{+} \to \mathbb{R}$ is the frequency response function mapping each non-negative eigenvalue $\{\lambda_1,\dots,\lambda_{|V|}\}$ of $\mathbf{L}$ to $\{h(\lambda_1),\dots,h(\lambda_{|V|})\}.$
\end{definition}

\begin{definition}[Graph convolution \cite{shuman2013emerging,shen2021powerful,xia2022fire}]\label{def:graph_convolution}
For a graph signal $\mathbf{x}$ and a graph filter $H(\mathbf{L})$, it follows that
\begin{equation}
\label{graph conv}
    H(\mathbf{L})\,\mathbf{x}
    \;=\;
    \mathbf{U} \,\mathrm{diag}\bigl(h(\lambda_1), \dots, h(\lambda_{|V|})\bigr)\,\mathbf{U}^T\,\mathbf{x},
\end{equation}
which first applies the GFT to $\mathbf{x}$, transforms the resulting spectrum by $h(\cdot)$, and finally performs the inverse GFT.
\end{definition}

In signal processing, signals are typically characterized by their smoothness and low-frequency components, whereas noise is usually non-smooth and dominates at high frequencies \cite{shen2021powerful}. In this context, a significant category of filters is LPFs, which enhance the smoothness of graph signals, thereby aiding noise reduction. We formally define the LPF as follows:
\label{app:LPF}
\begin{definition}
    (LPF) \cite{shen2021powerful,ramakrishna2020user,wai2019blind}: 
For $k=1,\cdots,|V|$ and $\lambda_1\le\cdots\le\lambda_{|V|}$, the graph filter $H(\mathbf{L})$ is $k$-low-pass if and only if $\eta_k \in [0,1]$, where
\begin{equation}
\eta_k:=\frac{max\{|h(\lambda_{k+1})|, \cdots, |h(\lambda_{|V|})|\}}{min\{|h(\lambda_{1})|, \cdots, |h(\lambda_{k})|\}}.
\end{equation}
\end{definition}

\subsection{Krylov Subspaces}
Krylov subspaces compactly capture the repeated action of a large matrix on a vector. Thus, for a user signal, they approximate the spectral properties of the item similarity matrix using the directions generated from that user, instead of requiring a full matrix decomposition. They are mathematically formalized as follows.
\begin{definition}{(\textit{Krylov subspace}~\cite{saad2011numerical,trefethen1997numerical})}
Given a matrix $\mathbf{M} \in \mathbb{R}^{d \times d}$ and a vector $\mathbf{r} \in \mathbb{R}^{d}$, the $K$-th order Krylov subspace generated by $(\mathbf{M}, \mathbf{r})$ is defined as
\begin{align}
\label{pre:krly}
\mathcal{K}_K(\mathbf{M},\mathbf{r}) = \text{span}\{\mathbf{r}, \mathbf{M}\mathbf{r}, \mathbf{M}^2\mathbf{r}, \ldots, \mathbf{M}^{K-1}\mathbf{r}\},
\end{align}
where $\mathbf{M}$ is the matrix of interest, $d$ is the dimensionality of the ambient space, and $\mathbf{r}$ is the initial vector from which the subspace is generated. Here, $K$ indicates the order of the subspace and $\text{span}\{\}$ denotes the set of all linear combinations of the given vectors.
\end{definition}

\section{Methodology}
\label{section 3}
This section begins by outlining the practical challenges associated with GF-based CF methods. We then provide a comprehensive description of the proposed \textsf{Mem-GF} method, along with its theoretical analyses.
\subsection{Challenges in GF-Based CF Methods}
\label{sec 3.1}
Conventional training-free GF-based CF methods \cite{shen2021powerful,xia2022fire,liu2023personalized,park2024turbo,xia2024hierarchical} in canonical recommendation tasks begin with constructing an item similarity graph, where each item is represented as a node and the similarities between items are modeled as edges. Its adjacency matrix is generally expressed as
\begin{equation}
\label{conventional_graph}
\begin{aligned}
    \tilde{\mathbf{P}} = \tilde{\mathbf{R}}^\top\tilde{\mathbf{R}} ; \tilde{\mathbf{R}} = \mathbf{D}^{\gamma}_{\mathcal{U}}\mathbf{R}\mathbf{D}^{1-\gamma}_\mathcal{I},
\end{aligned}
\end{equation}
where $|\mathcal{U}|$ is the number of users in a dataset; $\mathbf{R} \in \mathbb{R}^{|\mathcal{U}| \times |\mathcal{I}|}$ is the interaction matrix; $\tilde{\mathbf{R}}$ is the normalized interaction matrix; $\mathbf{D}_\mathcal{U}=\text{diag}(\mathbf{R}\mathbf{1})$ and $\mathbf{D}_\mathcal{I} = \text{diag}(\mathbf{1}^\top \mathbf{R})$; $\gamma \in [0,1]$ controls the normalization along users/items; and $\tilde{\mathbf{P}}$ is the adjacency matrix representing the item similarity graph. Meanwhile, the standard notion of GF involves performing full eigenvalue decompositions of the Laplacian or adjacency matrices representing the item similarity graph \cite{ortega2018graph}, which incurs a prohibitive computational complexity of \(O(|\mathcal{I}|^3)\). To prevent such computational issues, scalable GF-based CF methods exploit several techniques, including sparse matrix operations \cite{shen2021powerful}, top-\(k\) eigenvalue components \cite{shen2021powerful, xia2022fire}, polynomial graph filters \cite{park2024turbo}, and user clustering \cite{xia2024hierarchical}. Nevertheless, they still depend on the full-size item similarity matrix during filtering. For example, GF-based CF methods \cite{shen2021powerful,xia2022fire,liu2023personalized,xia2024hierarchical} typically leverage linear LPFs and/or ideal LPFs extracting only the top-\(k\) eigenvalue components instead of fully decomposing the item similarity matrix:
\begin{equation}
\label{gfcf}
    \mathbf{s}_u = \bigl(\tilde{\mathbf{P}} + \mu \mathbf{D}^{-1/2}_U\bar{\mathbf{U}}\bar{\mathbf{U}}^T\mathbf{D}^{-1/2}_I\bigr)\mathbf{r}_u,
\end{equation}
where $\mathbf{s}_u \in \mathbb{R}^{|\mathcal{I}|}$ is the predicted preferences for user $u$; ${\mathbf{r}}_{u}$ denotes the $u$-th row of $\mathbf{R}$, treated as a column graph signal for user $u$; $\bar{\mathbf{U}}\in \mathbb{R}^{|\mathcal{I}| \times k}$ is the top-$k$ singular vectors of $\tilde{\mathbf{R}}_0$; $\tilde{\mathbf{P}}$ is the linear LPF; $\mathbf{D}^{-1/2}_U\bar{\mathbf{U}}\bar{\mathbf{U}}^T\mathbf{D}^{-1/2}_I$ is the ideal LPF of $\tilde{\mathbf{P}}$; and $\mu$ is a hyperparameter balancing between the two filters. However, due to their limited expressiveness in designing various filters having different frequency response functions, recent studies proposed to use \textit{polynomial graph filters} to produce diverse frequency response functions using polynomial functions \cite{park2024turbo, shen2021powerful, qin2025polycf}. As representative work \cite{park2024turbo}, the prediction score for user $u$ via polynomial GF is formulated as
\begin{equation}
    \label{poly-gf}
    \mathbf{s}_{u} = \sum_{n=1}^N{a_{n}}\tilde{\mathbf{P}}^n\mathbf{r}_{u},
\end{equation}
where $\{a_n\}_{n=1}^N$ are the polynomial coefficients; $N$ is the maximum order of the matrix polynomial; and $\sum_{n=1}^N a_n \tilde{\mathbf{P}}^n$ represents the polynomial graph filter on the item similarity graph. It was proven that the polynomial graph filter $\sum_{n=1}^N{a_{n}}\tilde{\mathbf{P}}^n$ has a frequency response function of $\sum_{n=1}^N{a_{n}}(1-\lambda)^n$ \cite{park2024turbo,shen2021powerful,qin2025polycf}.

While conventional GF-based CF methods have adopted {\it user-specific batch} strategies (see \eqref{gfcf} and \eqref{poly-gf}), they still suffer from expensive costs for constructing or applying a \textit{universal} ({\it i.e.}, user-agnostic) graph filter across all users based on the large item similarity matrix $\tilde{\mathbf{P}}$ (see Fig. \ref{fig;intro}). For instance, a simple linear LPF (\textit{i.e.}, $\tilde{\mathbf{P}}$) having a frequency response function of $h(\lambda) = 1-\lambda$ has a dimension of $(|\mathcal{I}|, |\mathcal{I}|)$. Moreover, ideal LPFs or higher-order polynomial filters can further increase memory footprint and computation. As a result, devising a memory-efficient polynomial GF technique that can effectively operate under resource-constrained environments becomes a critical research challenge. This raises a natural question: \textit{``How can we optimize batch processing for each user to enhance memory efficiency in GF-based recommendations?"} To answer this question, we elaborate on the proposed \textsf{Mem-GF} method in the following subsection.

\subsection{Proposed Method: \textsf{Mem-GF}}
\label{sec:proposed_method}

To directly overcome the challenges outlined in Section~\ref{sec 3.1}, we propose \textsf{Mem-GF}, a memory-efficient CF method that leverages \textit{polynomial GF} without incurring the prohibitive memory costs of conventional GF-based CF approaches. The core idea is to use the \textit{Krylov subspace} for each user signal, thereby approximating matrix-vector operations without explicitly forming or storing large matrices such as $\tilde{\mathbf{P}}$. The Krylov subspace $\mathcal{K}_K(\tilde{\mathbf{P}}, \mathbf{r}_u)$ in \eqref{pre:krly}, defined as the span of vectors $\{\mathbf{r}_u, \tilde{\mathbf{P}}\mathbf{r}_u, \tilde{\mathbf{P}}^2\mathbf{r}_u, \dots,$ $\tilde{\mathbf{P}}^{K-1}\mathbf{r}_u\}$, captures progressively transformed directions of user $u$'s interaction vector under the repeated action of $\tilde{\mathbf{P}}$. By restricting computations to $\mathcal{K}_K(\tilde{\mathbf{P}},\mathbf{r}_u)$, we approximate the spectral properties of $\tilde{\mathbf{P}}$ in the user-specific Krylov subspace. To construct its orthonormal basis, we utilize the Lanczos algorithm~\cite{lanczos1950iteration}, which avoids direct eigenvalue/eigenvector computation. The schematic overview of the \textsf{Mem-GF} method is illustrated in Fig. \ref{fig:overview}. We summarize the pseudocode of \textsf{Mem-GF} in Algorithm \ref{alg:memgf}.
\subsubsection{Methodological Details}\label{sec:method_detail} \textsf{Mem-GF} first refines the normalized interaction matrix $\tilde{\mathbf{R}}$ by applying the Hadamard power, effectively regularizing the differences among elements in $\tilde{\mathbf{R}}$. Notably, unlike prior approaches \cite{park2024turbo, kim2025leveraging, park2025criteria} that apply the Hadamard power to the item similarity matrix $\tilde{\mathbf{P}}$, our method operates directly on $\tilde{\mathbf{R}}$ to avoid the need for explicit storage of $\tilde{\mathbf{P}}$:
\begin{equation}
  \label{eq:barR_def}
  \bar{\mathbf{R}} \;=\; \tilde{\mathbf{R}}^{\circ s},
\end{equation}
where $s$ is a hyperparameter for adjusting $\tilde{\mathbf{R}}$ and $\bar{\mathbf{R}}$ is the adjusted interaction matrix. Thus, the item similarity matrix is defined as $\tilde{\mathbf{P}}=\bar{\mathbf{R}}^\top \bar{\mathbf{R}}$ accordingly, but it is never explicitly stored. Next, let us turn to describing the Lanczos algorithm, which is central to \textsf{Mem-GF}. We construct a Krylov subspace $\mathcal{K}_K(\tilde{\mathbf{P}}, \mathbf{r}_u)$ with $K$ bases, where $\mathbf{r}_u$ represents user $u$'s interaction vector.\footnote{Note that \textsf{Mem-GF} is flexible in using the $u$-th row of $\bar{\mathbf{R}}$, denoted by $\bar{\mathbf{r}}_u$, as graph signals.} This process generates a user-specific orthonormal Krylov basis $\mathbf Q_u\in\mathbb R^{|\mathcal I|\times K}$ and a low-dimensional tridiagonal matrix $\mathbf T_u\in\mathbb R^{K\times K}$. The matrix $\mathbf T_u$ represents the projection of $\tilde{\mathbf P}$ onto the Krylov subspace:
\begin{equation}
\label{eq:qtq}
\mathbf{T}_u = \mathbf{Q}_u^\top \tilde{\mathbf{P}}\mathbf{Q}_u.
\end{equation}
More precisely, finite-step Lanczos may involve a residual term outside the constructed subspace; however, the polynomial action on the starting vector $\mathbf r_u$ is exactly represented for degrees $N<K$ under exact arithmetic, as formalized in Theorem~\ref{thm:polynomial_approx}.

\begin{figure}[t]
    \centering
    \includegraphics[width=1\columnwidth]{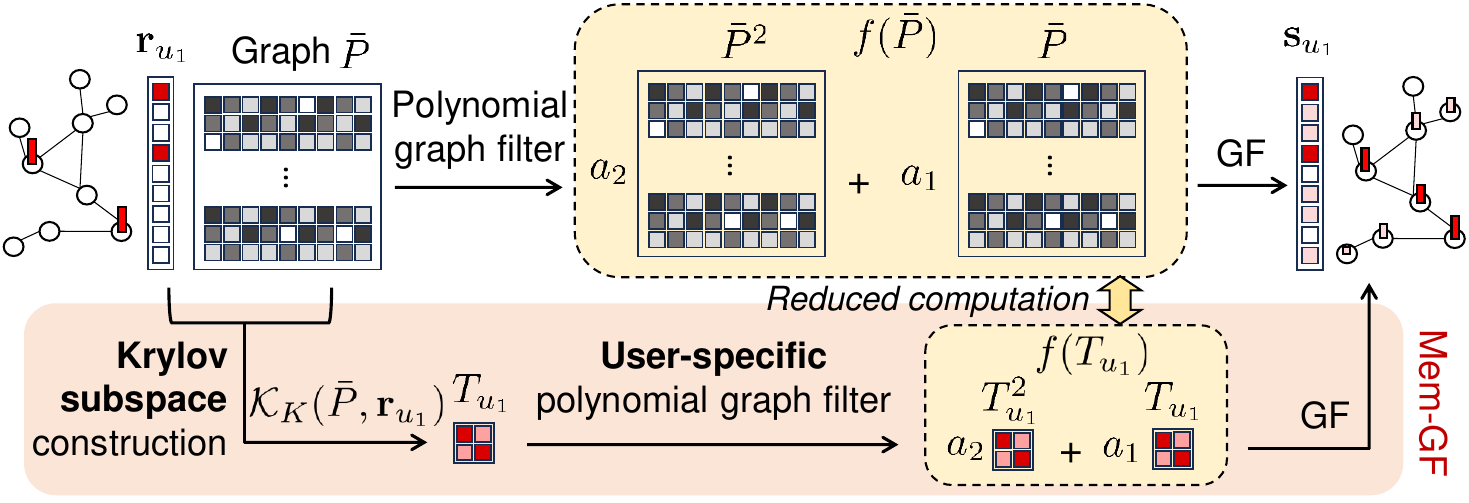}
    \vspace{-3mm}
    \caption{The schematic overview of \textsf{Mem-GF}.}
    \label{fig:overview}
\end{figure}
\subsubsection{Core Idea of \textsf{Mem-GF}} 

\begin{algorithm}[t]
\caption{\textsf{Mem-GF}}
\label{alg:memgf}
\begin{algorithmic}[1]
  \renewcommand{\algorithmicrequire}{\textbf{Input:}}
  \renewcommand{\algorithmicensure}{\textbf{Output:}}
  \REQUIRE interaction matrix $\mathbf{R} \in \mathbb{R}^{|\mathcal{U}|\times|\mathcal{I}|}$, set of users $\mathcal{U}$, set of items $\mathcal{I}$
  \ENSURE $\{\mathbf{s}_u\}_{u \in \mathcal{U}}$
  \STATE $\tilde{\mathbf{R}} = \mathbf{D}^{\gamma}_{\mathcal{U}}\mathbf{R}\mathbf{D}^{1-\gamma}_\mathcal{I}$
  \STATE $\bar{\mathbf{R}} = \tilde{\mathbf{R}}^{\circ s}$
  \FOR{each $u \in \mathcal{U}$}
    \STATE /* Lanczos iteration */
    \STATE $\beta_0 \gets 0,\ \mathbf{q}_{u,0} \gets \mathbf{0},\ \mathbf{q}_{u,1} \gets \mathbf{r}_u / \|\mathbf{r}_u\|_2$
    \FOR{$j \gets 1$ to $K$}
      \STATE $\mathbf{w}_u \gets \bar{\mathbf{R}}^\top(\bar{\mathbf{R}}\,\mathbf{q}_{u,j}) - \beta_{j-1}\,\mathbf{q}_{u,j-1}$
      \STATE $\alpha_j \gets \mathbf{q}_{u,j}^\top \mathbf{w}_u$
      \STATE $\mathbf{w}_u \gets \mathbf{w}_u - \alpha_j\,\mathbf{q}_{u,j}$
      \STATE $\beta_j \gets \|\mathbf{w}_u\|_2$
      \IF{$\beta_j \neq 0$}
        \STATE $\mathbf{q}_{u,j+1} \gets \mathbf{w}_u / \beta_j$
      \ENDIF
    \ENDFOR
    \STATE Form $\mathbf{Q}_u$ with $\mathbf{q}_{u,1},\ldots,\mathbf{q}_{u,K}$ and $\mathbf{T}_u$ with $\{\alpha_j,\beta_j\}$
    \STATE $f(\mathbf{T}_u) \gets \sum_{n=0}^N a_n\,(\mathbf{T}_u)^n$
    \STATE $\mathbf{s}_u \gets \|\mathbf{r}_u\|_2\;\mathbf{Q}_u\;f(\mathbf{T}_u)\;\mathbf{e}_1$
  \ENDFOR
  \RETURN $\{\mathbf{s}_u\}_{u \in \mathcal{U}}$
\end{algorithmic}
\end{algorithm}
The Lanczos algorithm \cite{lanczos1950iteration} proceeds as follows:

\begin{itemize}
    \item \textbf{Initialization}: The algorithm begins with the normalized vector $\mathbf{q}_{u,1} = \mathbf{r}_u / \| \mathbf{r}_u \|_2$, where $\mathbf{r}_u$ is the interaction vector for user $u$. Additional parameters, including $\beta_0 = 0$ and $\mathbf{q}_{u,0} = \mathbf{0}$, required for iterations are also initialized.
    \item \textbf{Iteration}: At each iteration $j$, the algorithm computes the projection of $\tilde{\mathbf{P}}$ onto the current basis:
\begin{align*}
    \mathbf{w}_u &= \tilde{\mathbf{P}}\mathbf{q}_{u,j} - \beta_{j-1}\mathbf{q}_{u,j-1}, \\
    \alpha_j &= \mathbf{q}_{u,j}^\top \mathbf{w}_u.
\end{align*}
The residual $\mathbf{w}_u$ is then orthogonalized and normalized:
\begin{align*}
    \mathbf{w}_u &\leftarrow \mathbf{w}_u - \alpha_j \mathbf{q}_{u,j}, \\
    \beta_j &= \|\mathbf{w}_u\|_2, \\
    \mathbf{q}_{u,j+1} &= \mathbf{w}_u / \beta_j \quad \text{(if } \beta_j \neq 0\text{)}.
\end{align*}

  \item \textbf{Termination}: The iterations continue until $j = K$ or $\beta_j = 0$, while constructing the tridiagonal matrix $\mathbf{T}_u$ using the coefficients $\{\alpha_j\}_{j=1}^K$ and $\{\beta_j\}_{j=1}^{K-1}$:  
\begin{equation}
\mathbf{T}_u = \begin{bmatrix}
\alpha_1 & \beta_1 & 0 & \cdots & 0 \\
\beta_1 & \alpha_2 & \beta_2 & \ddots & \vdots \\
0 & \beta_2 & \alpha_3 & \ddots & 0 \\
\vdots & \ddots & \ddots & \ddots & \beta_{K-1} \\
0 & \cdots & 0 & \beta_{K-1} & \alpha_K \\
\end{bmatrix},
\end{equation}
\end{itemize}
where $\mathbf{q}_{u,j} \in \mathbb{R}^{|\mathcal{I}|}$ is the $j$-th vector in the Lanczos basis for user $u$; $\alpha_j$ represents the local projection coefficient of $\tilde{\mathbf{P}}\,\mathbf{q}_{u,j}$ onto $\mathbf{q}_{u,j}$; and $\beta_j$ is the norm of the resulting residual, used to orthogonalize $\mathbf{q}_{u,j+1}$ \textit{w.r.t.} the previous basis vectors. It is worth noting that, instead of directly storing $\tilde{\mathbf{P}}$, \textsf{Mem-GF} calculates the matrix-vector product $\tilde{\mathbf{P}}\mathbf{Q}_u$ using only the adjusted interaction matrix $\bar{\mathbf{R}}$ as $\tilde{\mathbf{P}}\mathbf{Q}_u = \bar{\mathbf{R}}^\top (\bar{\mathbf{R}}\mathbf{Q}_u)$. Since $\bar{\mathbf{R}} \in \mathbb{R}^{|\mathcal{U}| \times |\mathcal{I}|}$ is much smaller and sparser than the full item similarity matrix $\tilde{\mathbf{P}} = \bar{\mathbf{R}}^{\top}\bar{\mathbf{R}}$, this reduces the memory requirement from storing a full-size item--item matrix to storing only $\bar{\mathbf{R}}$ and the compact Krylov quantities, making \textsf{Mem-GF} viable for large-scale datasets (refer to Figs. \ref{fig;intro} and \ref{fig:overview}).

To see how \textsf{Mem-GF} computes polynomial graph filters, note that $\mathbf Q_u$ forms a user-specific Krylov subspace, while $\mathbf T_u$ is the projected representation of $\tilde{\mathbf P}$ in that subspace. By the Lanczos initialization, we have $\mathbf{Q}_u^\top\mathbf{r}_u = \|\mathbf{r}_u\|_2 \mathbf{e}_1$. Therefore, for a polynomial graph filter $f(\tilde{\mathbf P})=\sum_{n=0}^{N}a_n\tilde{\mathbf P}^n$ with $N<K$, the Krylov representation gives
\begin{equation}
\label{eq_signal_1}
    f(\tilde{\mathbf{P}})\mathbf{r}_u
    =
    \|\mathbf{r}_u\|_2\,\mathbf{Q}_u
    \left(\sum_{n=0}^N a_n\mathbf{T}_u^n\right)\mathbf{e}_1,
\end{equation}
as formally shown in Theorem~\ref{thm:polynomial_approx}. This relation enables polynomial graph filtering without explicitly forming $\tilde{\mathbf P}$.

Thus, our \textsf{Mem-GF} method is finally formulated as follows:
\begin{equation}
\label{eq:final_score}
\mathbf{s}_u = \|\mathbf{r}_u\|_2 \mathbf{Q}_u f(\mathbf{T}_u) \mathbf{e}_1,
\end{equation}
where $f(\mathbf{T}_u)$ represents the user-specific polynomial graph filter in the Krylov
subspace and is expressed as
\begin{equation}
\label{eq:poly_filter}
f(\mathbf{T}_u) = \sum_{n=0}^N a_n \mathbf{T}_u^n.
\end{equation}
This formulation ensures that both the graph filter and the signal are \textbf{user-specific}, enabling personalized and efficient recommendation by allocating memory in a batch-based manner for each user's operations. By reducing the polynomial graph filter computation into {\it smaller, user-specific batches}, our method is capable of further optimizing memory usage, ensuring scalability even with limited hardware resources. We next highlight additional advantages offered by the proposed \textsf{Mem-GF} method.
\begin{remark}
First, the user-specific matrices $\mathbf{Q}_u$ and $\mathbf{T}_u$ can be pre-computed and cached, enabling near-instantaneous inference through lightweight matrix multiplications. This design facilitates deployment in real-world systems with stringent latency requirements. Second, once the Krylov basis $(\mathbf{Q}_u,\,\mathbf{T}_u)$ is constructed through $K$ Lanczos iterations, an arbitrary polynomial graph filter of order $K' < K$ can be calculated by utilizing only the first $K'$ Krylov basis vectors, without further access to the interaction matrix $\bar{\mathbf{R}}$. Consequently, validation-based filter selection becomes much cheaper than recomputing each filter independently. Such capability is especially valuable in personalized recommendation scenarios, as optimal polynomial filters can vary substantially across individual users \cite{liu2023personalized,park2025criteria}.
\end{remark}

Next, we are interested in analyzing how many Krylov subspace bases are sufficient to guarantee {\it no approximation error}. To address this, we provide a theoretical analysis by establishing the following theorem.

\begin{theorem}[Zero-Error Guarantees in Krylov Subspaces]
\label{thm:polynomial_approx}
Let $\tilde{\mathbf{P}}\in\mathbb{R}^{|\mathcal{I}|\times|\mathcal{I}|}$ be a symmetric matrix ({\it e.g.}, an item similarity matrix), and let $\mathbf{r}_u\in\mathbb{R}^{|\mathcal{I}|}$ be the signal for user $u$. Run $K$ iterations of the Lanczos algorithm on the pair $(\tilde{\mathbf P},\mathbf r_u)$, yielding an orthonormal Krylov basis $\mathbf Q_u\in\mathbb R^{|\mathcal I|\times K}$ and a tridiagonal projected matrix $\mathbf T_u=\mathbf Q_u^\top\tilde{\mathbf P}\mathbf Q_u\in\mathbb R^{K\times K}$. Consider a polynomial graph filter $f(\tilde{\mathbf{P}}) \;=\;\sum_{n=0}^{N} a_n \,\tilde{\mathbf{P}}^n$ of degree $N.$ If $N < K$, then the Lanczos algorithm in \textsf{Mem-GF} incurs no approximation error under exact arithmetic. 
\end{theorem}

\begin{proof}
The Lanczos procedure constructs $\mathbf Q_u=[\mathbf q_{u,1},\ldots,\mathbf q_{u,K}]$ with $\mathbf q_{u,1}=\mathbf r_u/\|\mathbf{r}_u\|_2$ and the projected matrix $\mathbf T_u=\mathbf Q_u^\top\tilde{\mathbf P}\mathbf Q_u$. By the Krylov construction, for every $0\le n\le N<K$, the vector $\tilde{\mathbf P}^n\mathbf r_u$ belongs to $\mathcal K_K(\tilde{\mathbf P},\mathbf r_u)$ and can be represented in the basis $\mathbf Q_u$ as
\[
   \tilde{\mathbf{P}}^n\mathbf{r}_u
   =
   \|\mathbf{r}_u\|_2\mathbf{Q}_u\mathbf{T}_u^n\mathbf{e}_1.
\]
Therefore,
\[
\begin{aligned}
   f(\tilde{\mathbf{P}})\mathbf{r}_u
   &=
   \sum_{n=0}^{N} a_n \tilde{\mathbf{P}}^n\mathbf{r}_u  \\
   &=
   \|\mathbf{r}_u\|_2\mathbf{Q}_u
   \left(\sum_{n=0}^{N}a_n\mathbf{T}_u^n\right)\mathbf{e}_1
   =
   \|\mathbf{r}_u\|_2\mathbf{Q}_u f(\mathbf{T}_u)\mathbf{e}_1.
\end{aligned}
\]
Thus, for all $N<K$, the equality holds with no approximation error. This completes the proof of Theorem \ref{thm:polynomial_approx}.
\end{proof}

 Theorem~\ref{thm:polynomial_approx} implies that, whenever the polynomial degree $N$ is less than the Krylov subspace size $K$, \textsf{Mem-GF} exactly reproduces $f(\tilde{\mathbf P})\mathbf r_u$ in the original space under exact arithmetic. Thus, setting $K=N+1$ is sufficient to avoid approximation error in the algebraic Krylov representation. In practice, finite-precision computation or choosing $N\ge K$ may affect this exactness; therefore, we tune $K$ and $N$ on the validation set. Separately, severe sparsity or noise in user signals may degrade empirical recommendation quality rather than invalidate the algebraic guarantee itself. We provide experimental verification of this theoretical claim in Section \ref{app:exp_thm}.

\subsubsection{High-Order Polynomial Graph Filter Designs}
Conventional polynomial GF-based CF methods~\cite{park2024turbo,kim2025leveraging,park2025criteria} restrict polynomial filters up to third-order terms to bypass the problem of memory constraints. However, explicitly constructing and storing $\tilde{\mathbf{P}}$ or computing higher-order powers ({\it e.g.}, $\tilde{\mathbf{P}}^3$) still requires memory scaling as $\mathcal{O}(|\mathcal{I}|^2)$, which is infeasible for large sets of items. In contrast, \textsf{Mem-GF} overcomes the memory bottleneck issue using the user-specific tridiagonal matrix \(\mathbf{T}_u\) in \eqref{eq:qtq}. Since \(K \ll |\mathcal{I}|\) for $\mathbf{T}_u \in \mathbb{R}^{K\times K}$, it becomes computationally trivial to raise \(\mathbf{T}_u\) to higher powers, thereby allowing the order of polynomial filters to be significantly higher.

\begin{figure}[t]
    \centering
    \includegraphics[width=1\columnwidth]{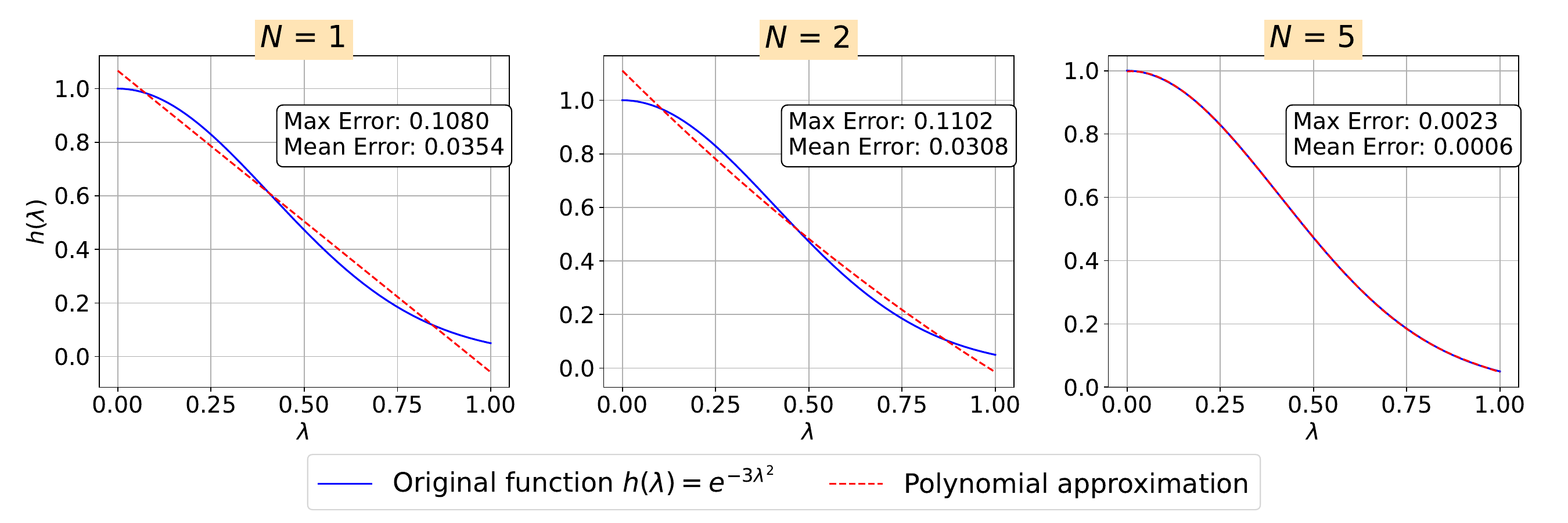}
   \caption{Comparison of the function \( h(\lambda) = e^{-\gamma \lambda^2} \) and its polynomial approximations for different polynomial orders (\( N = 1, 2, 5 \)). The maximum and mean errors are annotated within each subplot to indicate approximation accuracy.}
    \label{fig;polyfit}
\end{figure}
This flexibility allows \textsf{Mem-GF} to approximate a broad spectrum of target frequency responses. Specifically, we consider employing graph filters having more complex frequency responses. This leads us to propose a new polynomial graph filter that approximates the Gaussian filter \cite{ito2000gaussian, li2022g} that is widely used in image and graph domains, whose frequency response is expressed as $h(\lambda) = e^{-\tau \lambda^2}$, where $\tau$ is the damping factor in the Gaussian filter.\footnote{In addition to the Gaussian filter, any sophisticated filters can be designed by appropriately adjusting the polynomial coefficients.} Due to the fact that the polynomial graph filter $\sum_{n=1}^N{a_{n}}\tilde{\mathbf{P}}^n$ has a frequency response function of $\sum_{n=1}^N{a_{n}}(1-\lambda)^n$ \cite{park2024turbo, shen2021powerful,qin2025polycf}, we are ready to numerically solve a non-linear least squares (LS) problem to find the polynomial coefficients $\{a_n\}_{n=1}^N$ in \eqref{eq:poly_filter} within the maximum polynomial order $N$, as in \cite{park2024turbo}. As illustrated in Fig.~\ref{fig;polyfit}, performing higher-order polynomial approximations ({\it e.g.}, \(N = 5\)) significantly reduces the error compared to lower-order polynomials ({\it e.g.}, \(N = 1, 2\)). This demonstrates the clear advantage of high-order polynomials in closely approximating the original function \(h(\lambda) = e^{-\tau \lambda^2}\). Therefore, by setting $N = 5$ and $\tau = 3$,\footnote{One can arbitrarily set the hyperparameter $\tau$. Here, increasing $\tau$ results in the further utilization of the low-pass frequency components, and vice versa.} we are capable of discovering a new \textit{predefined} polynomial graph filter having a frequency response function as $h(\lambda) = e^{-3 \lambda^2}$ as follows:
\begin{equation}
\label{5th_pgf}
    f(\mathbf{T}_u) =  8\mathbf{T}_u + 3.1\mathbf{T}_u^2 -11.2\mathbf{T}_u^3 +6.2\mathbf{T}_u^4-1.1\mathbf{T}_u^5.
\end{equation}
Finally, in our study, we utilize three predefined LPFs used in \cite{park2024turbo, kim2025leveraging} as well as our newly devised fifth-order polynomial filter, which are listed as follows:
\begin{itemize}
    \item \textsf{Mem-GF-1} ($f(\mathbf{T}_u) = \mathbf{T}_u$): The linear LPF with the first polynomial order, whose frequency response is $h(\lambda) = 1-\lambda$.
    \item \textsf{Mem-GF-2} ($f(\mathbf{T}_u) = 2\mathbf{T}_u-\mathbf{T}_u^2$): The second-order LPF, whose frequency response is $h(\lambda) = 1-\lambda^2$
    \item \textsf{Mem-GF-3} ($f(\mathbf{T}_u) = \mathbf{T}_u + 0.01({-\mathbf{T}_u^3+10\mathbf{T}_u^2-29\mathbf{T}_u})$): The third-order LPF which approximates ideal LPF (\textit{i.e.}, $h(\lambda) = \mathbf{1}_{\lambda\le p}$ with cutoff frequency $p$) \cite{park2024turbo}.
    \item \textsf{Mem-GF-5} ($f(\mathbf{T}_u) = 8\mathbf{T}_u + 3.1\mathbf{T}_u^2 -11.2\mathbf{T}_u^3 +6.2\mathbf{T}_u^4-1.1\mathbf{T}_u^5$): The fifth-order LPF, whose frequency response approximates the Gaussian filter $h(\lambda) = e^{-3 \lambda^2}$.
\end{itemize}
The optimal filter and hyperparameters of \textsf{Mem-GF} can be found via hyperparameter tuning on the validation set. Although the aforementioned four graph filters are utilized in our experiments, higher-order graph filters can be utilized and any arbitrary functions can be designed via the same procedures.

\subsubsection{Complexity Analysis} 
\label{sec 3.2.4}
We first theoretically analyze the space (\textit{i.e.}, memory) complexity of \textsf{Mem-GF} during preprocessing stage. A central advantage of \textsf{Mem-GF} is its ability to avoid explicitly forming the item similarity matrix
\(\tilde{\mathbf{P}} = \bar{\mathbf{R}}^\top \bar{\mathbf{R}}\in\mathbb{R}^{|\mathcal{I}|\times |\mathcal{I}|}\). Storing \(\tilde{\mathbf{P}}\) would na\"ively demand memory of
\(\mathcal{O}(|\mathcal{I}|^2)\), which can be reduced to $\mathrm{nnz}(\tilde{\mathbf{P}})$ in sparse format. However, $\tilde{\mathbf{P}}$ in \eqref{conventional_graph} is often far denser than $\tilde{\mathbf{R}}$, as each element contains two-hop item-neighbor information. Instead, \textsf{Mem-GF}
stores only $\bar{\mathbf{R}}$, $\mathbf{Q}_u$, and $\mathbf{T}_u$ (see \eqref{eq:barR_def} and \eqref{eq:qtq}). First, the adjusted interaction matrix in a sparse format has the dimension of \(\bar{\mathbf{R}}\in\mathbb{R}^{|\mathcal{U}|\times |\mathcal{I}|}\). Thus, its space requirement is
    \(\mathcal{O}\bigl(\mathrm{nnz}(\mathbf{R})\bigr)\). Since \(\bar{\mathbf{R}}\) is obtained by a Hadamard power of \(\tilde{\mathbf{R}}\), it shares the same sparsity pattern as \(\mathbf{R}\), implying that \(\mathrm{nnz}(\bar{\mathbf{R}}) = \mathrm{nnz}(\mathbf{R})\).
Next, we need to store temporary user-specific Lanczos vectors \(\mathbf{Q}_u\in\mathbb{R}^{|\mathcal{I}|\times K}\) and the tridiagonal matrix
    \(\mathbf{T}_u\in\mathbb{R}^{K\times K}\), which have the dimension linear in \(|\mathcal{I}|K\) and \(K^2\), respectively. Importantly, from the fact that
    they are constructed \emph{per user (or per batch)}, the algorithm can overwrite or free this memory once each 
    user's final predictions are computed. Hence, the overall memory footprint is finally given by
\begin{equation}
    \mathcal{O}\bigl(\mathrm{nnz}(\mathbf{R}) + |\mathcal{I}|K + K^2),
\end{equation}
rather than \(\mathcal{O}(|\mathcal{I}|^2)\) or \(\mathcal{O}(\mathrm{nnz}(\tilde{\mathbf{P}}))\). This design substantially alleviates the memory bottleneck, making it
feasible to handle large-scale datasets on standard hardware resources.

Next, we delve into theoretically analyzing the time (\textit{i.e.}, computational) complexity of \textsf{Mem-GF} during preprocessing. The computational cost of \textsf{Mem-GF} can be broken down into two core stages: 
(i) constructing the Krylov subspace for each user via the Lanczos algorithm and 
(ii) performing polynomial GF in the reduced space.

\begin{enumerate}[leftmargin=*]
    \item \emph{Krylov subspace construction (Lanczos algorithm).} For user \(u\), let 
    \(\mathbf{r}_u \in \mathbb{R}^{|\mathcal{I}|}\) be the interaction vector. We conduct \(K\) iterations of the Lanczos 
    algorithm to build the matrix \(\mathbf{Q}_u\in\mathbb{R}^{|\mathcal{I}|\times K}\) and the tridiagonal matrix \(\mathbf{T}_u\in\mathbb{R}^{K\times K}\).
    Each iteration requires a matrix-vector multiplication with \(\tilde{\mathbf{P}}=\bar{\mathbf{R}}^\top\bar{\mathbf{R}}\). Rather than forming
    \(\tilde{\mathbf{P}}\) explicitly, we compute $\tilde{\mathbf{P}}\,\mathbf{q}_{u,j} \;=\; \bar{\mathbf{R}}^\top\!\bigl(\bar{\mathbf{R}}\,\mathbf{q}_{u,j}\bigr)$.
    Each multiplication with \(\bar{\mathbf{R}}\) or \(\bar{\mathbf{R}}^\top\) can be executed in
    \(\mathcal{O}\!\bigl(\mathrm{nnz}(\bar{\mathbf{R}})\bigr)\) time by exploiting sparsity. Consequently, one Lanczos
    iteration costs \(\mathcal{O}\!\bigl(\mathrm{nnz}(\bar{\mathbf{R}})\bigr)\). Over \(K\) steps, the cost per user is
    \(\mathcal{O}(K\,\mathrm{nnz}(\bar{\mathbf{R}}))\), which results in
    \(\mathcal{O}(|\mathcal{U}|\,K\,\mathrm{nnz}(\bar{\mathbf{R}}))\) for $|\mathcal{U}|$ total users.

    \item \emph{Polynomial GF and score reconstruction.} Given \(\mathbf{Q}_u\) and \(\mathbf{T}_u\), we next compute a polynomial graph filter
\( f(\mathbf{T}_u) = \sum_{n=0}^N a_n\,\mathbf{T}_u^n \)
in the reduced space. Specifically, the prediction score $\mathbf{s}_u$ is
\(\|\mathbf{r}_u\|_2\,\mathbf{Q}_u\,f(\mathbf{T}_u)\,\mathbf{e}_1 \in \mathbb{R}^{|\mathcal{I}|}\).
Because \(\mathbf{T}_u\) is of size \(K\times K\), powering \(\mathbf{T}_u\) up to order \(N\) typically takes
\(\mathcal{O}(K^2)\) up to \(\mathcal{O}(K^3)\) depending on implementation \cite{saad2003iterative}. In practice, since \(K\) is small 
(\textit{e.g.}, $K=N+1$ due to Theorem \ref{thm:polynomial_approx}), this step is minor.  
Finally, multiplying \(\mathbf{Q}_u f(\mathbf{T}_u)\mathbf{e}_1\) back into the original dimension
costs \(\mathcal{O}(|\mathcal{I}|K)\) per user. Aggregated across all users, the reconstruction needs $\mathcal{O}\bigl(|\mathcal{U}|\;|\mathcal{I}|\;K\bigr)$.
\end{enumerate}
Therefore, the total time complexity of \textsf{Mem-GF} is given by
\begin{equation}
    \mathcal{O}\Bigl(|\mathcal{U}|\,K\,\bigl(\mathrm{nnz}(\mathbf{R})+|\mathcal{I}|\bigr)\Bigr).
\end{equation}


\section{Experimental Evaluations}
\label{section 4}

In this section, we systematically conduct extensive experiments to address the key research questions (RQs) outlined below:

\begin{itemize}    
    \item \textbf{RQ1:} How memory-efficient is \textsf{Mem-GF} during preprocessing compared to benchmark GF-based recommendation methods?
    \item \textbf{RQ2:} How efficient is \textsf{Mem-GF} during inference in terms of runtime and memory consumption compared to benchmark GF-based recommendation methods?
    \item \textbf{RQ3:} How does the memory on \textsf{Mem-GF} scale {\it w.r.t.} the number of users, items, and interactions?
    \item \textbf{RQ4:} How much does \textsf{Mem-GF} improve recommendation accuracy over benchmark recommendation methods?
    \item \textbf{RQ5:} How does the number of Krylov subspace bases affect the performance of \textsf{Mem-GF}?
    \item \textbf{RQ6:} How do types of polynomial filters affect the performance of \textsf{Mem-GF}?
    \item \textbf{RQ7:} How sensitive is \textsf{Mem-GF} to variations in key hyperparameters?
\end{itemize}

\vspace{-2mm}
\subsection{Experimental Settings}
\begin{table}[t]
\footnotesize
  \captionsetup{skip=2pt}
  \caption{The statistics of three datasets.}
  \begin{tabular}{ccccccl}
    \toprule
    Dataset & \# of users & \# of items & \# of interactions & Density \\ 
    \midrule
    Yelp& 31,668 & 38,048 & 1,561,406 &  0.130\%\\
    Amazon-book & 52,643 & 91,599& 2,984,108 & 0.062\% \\
    ML-20M & 138,493 & 26,744 & 20,000,263 & 0.540\% \\
  \bottomrule
\end{tabular}
\vspace{0mm}
\label{table:datasets}
\end{table}
\noindent\textbf{Datasets.} We carry out experiments on the three benchmark datasets widely adopted for evaluating the performance of recommender systems for CF, which include Yelp, Amazon-book, and MovieLens-20M (ML-20M) \cite{berg2017graph, ying2018graph,zheng2018spectral,wang2020disentangled,wang2019neural,he2020lightgcn,he2017neural,he2016vbpr}. Notably, ML-20M is a large-scale, real-world dataset, consisting of 20 million user interactions. The statistics of the three datasets are summarized in Table \ref{table:datasets}.

\noindent\textbf{Competitors.} We compare \textsf{Mem-GF} against twenty-one benchmark CF methods, which can be grouped into two categories. The first group comprises canonical benchmark CF methods built on commonly used architectures, including MF-based methods (MF-BPR~\cite{rendle2009bpr} and NeuMF~\cite{he2017neural}), autoencoder-based methods (Multi-DAE, Multi-VAE~\cite{liang2018variational}, RecVAE~\cite{shenbin2020recvae}, and SVD-AE~\cite{hong2024svd}), GCN-based methods (NGCF~\cite{wang2019neural}, LightGCN~\cite{he2020lightgcn}, DGCF~\cite{wang2020disentangled}, SGL~\cite{wu2021self}, SimpleX \cite{mao2021simplex}, and SGDE \cite{peng2024less}), generative model-based methods (CFGAN~\cite{chae2018cfgan}, DiffRec, L-DiffRec~\cite{wang2023diffusion}, and FlowCF~\cite{liu2025flow}), and link propagation-based method (LinkProp and LinkProp-Multi \cite{fu2022revisiting}). The second group consists of recent GF-based CF methods, including GF-CF~\cite{shen2021powerful}, PGSP~\cite{liu2023personalized}, FPSR~\cite{wei2023fine}, Turbo-CF~\cite{park2024turbo}, and HiGSP~\cite{xia2024hierarchical}. Certain recent GF-based approaches, including PolyCF \cite{qin2025polycf} and LanczosNet \cite{liao2019lanczosnet}, were excluded from our empirical comparison, as the source code of PolyCF is not publicly available and LanczosNet was not specifically designed for CF tasks with recommendation datasets involving tens of thousands of nodes.

\noindent\textbf{Evaluation protocols.} For the Yelp and Amazon-book datasets, we follow the same data splits and performance as those in prior studies \cite{he2020lightgcn,wang2019neural,park2024turbo,liu2023personalized}, directly using each paper's reported best performance for consistency. For ML-20M, we preprocess the dataset by splitting interactions into a ratio of 80:20 for training and testing, and then further reserve 10\% of the training set as a validation set under a 10-core setting ({\it i.e.}, filter out users with fewer than ten interactions). To assess the performance of the top-$K$ recommendations, we use widely adopted benchmark metrics from the literature \cite{berg2017graph, ying2018graph,zheng2018spectral,wang2020disentangled,wang2019neural,he2020lightgcn,he2017neural,he2016vbpr}, including \textit{recall} and \textit{normalized discounted cumulative gain (NDCG)}, with $K$ set to 20 by default. For each metric, we report the average results over 10 independent runs, except for deterministic methods \cite{steck2019embarrassingly, shen2021powerful,liu2023personalized,park2024turbo,hong2024svd,xia2024hierarchical} (including \textsf{Mem-GF}). Unless otherwise specified, all runtime reports for GF-based methods refer to their preprocessing time.

\begin{table*}[!t]\centering
\footnotesize
\captionsetup{skip=2pt}
\caption{Comparison of memory (in GB) and time (in seconds) consumption during preprocessing among seven training-free GF-based recommendation methods across three datasets. RAM and VRAM values represent the peak memory usage recorded during execution. Entries marked as OOM denote out-of-memory errors, and entries marked as ``--'' denote unavailable GPU measurements caused by GPU buffer OOM or CPU-only fallback. The best and second-best performers are marked in \textbf{bold} and \underline{underlined}, respectively. The performance of Turbo-CF and \textsf{Mem-GF} is evaluated according to the maximum order of the polynomial graph filters used.}
\label{resource_consumption}
\begin{tabular}{c|c|c|c|c|c|c|>{\centering\arraybackslash}p{4.8mm}>{\centering\arraybackslash}p{4.8mm}>{\centering\arraybackslash}p{4.8mm}|>{\centering\arraybackslash}p{4.8mm}>{\centering\arraybackslash}p{4.8mm}>{\centering\arraybackslash}p{4.8mm}>{\centering\arraybackslash}p{4.8mm}}
\toprule[1pt]
\multirow{2}{*}{\textbf{Dataset}}
& \multirow{2}{*}{\textbf{Method}}
& \multirow{2}{*}{EASE}
& \multirow{2}{*}{GF-CF}
& \multirow{2}{*}{PGSP}
& \multirow{2}{*}{HiGSP}
& \multirow{2}{*}{FPSR}
& \multicolumn{3}{c|}{Turbo-CF}
& \multicolumn{4}{c}{\textsf{Mem-GF}}\\
\cmidrule(lr){8-10}\cmidrule(lr){11-14}
& & & & & & & 1 & 2 & 3 & 1 & 2 & 3 & 5\\
\midrule[1pt]
\multirow{5}{*}{\textbf{Yelp}}
& \textbf{VRAM (A)}
  & 31.2 & 10.9 & --& --& 9.3 & 9.2 & 25.7 & 25.7 & 1.9 & 2.4 & 2.7 & 3.3\\
& \textbf{RAM (B)}
  & 0.8 & 2.1 & 74.1 & OOM & 0.7 & 0.7 & 0.7 & 0.7 & 0.7 & 0.7 & 0.7 & 0.7\\
& \cellcolor[gray]{0.9}\textbf{Total memory (A)+(B)}
  & \cellcolor[gray]{0.9}32.0 & \cellcolor[gray]{0.9}13.0 & \cellcolor[gray]{0.9}74.1 & \cellcolor[gray]{0.9}OOM
  & \cellcolor[gray]{0.9}10.0 & \cellcolor[gray]{0.9}\underline{9.9} & \cellcolor[gray]{0.9}26.4 & \cellcolor[gray]{0.9}26.4
  & \cellcolor[gray]{0.9}\textbf{2.6} & \cellcolor[gray]{0.9}3.1 & \cellcolor[gray]{0.9}3.4 & \cellcolor[gray]{0.9}4.0\\
& \cellcolor[gray]{0.9}\textbf{Time}
  & \cellcolor[gray]{0.9}15.68 & \cellcolor[gray]{0.9}7.0 & \cellcolor[gray]{0.9}1398.0 & \cellcolor[gray]{0.9}--  & \cellcolor[gray]{0.9}58.7 & \cellcolor[gray]{0.9}\underline{4.2} & \cellcolor[gray]{0.9}140.2 & \cellcolor[gray]{0.9}420.7
  & \cellcolor[gray]{0.9}\textbf{1.5} & \cellcolor[gray]{0.9}2.9 & \cellcolor[gray]{0.9}4.3 & \cellcolor[gray]{0.9}6.5\\
& \textbf{GPU OOM}
  &  &  & \checkmark & \checkmark &  &  &  &  &  &  &  & \\
\midrule[1pt]
\multirow{5}{*}{\textbf{Amazon-book}}
& \textbf{VRAM (A)}
  & --& 34.5 & --& --& 32.4 & 45 & --& --& 4.4 & 5.2 & 5.9 & 7.3\\
& \textbf{RAM (B)}
  & OOM & 0.7 & OOM & OOM & 2.0 & 0.7 & OOM & OOM & 0.7 & 0.7 & 0.7 & 0.7\\
& \cellcolor[gray]{0.9}\textbf{Total memory (A)+(B)}
  & \cellcolor[gray]{0.9}OOM & \cellcolor[gray]{0.9}35.2 & \cellcolor[gray]{0.9}OOM & \cellcolor[gray]{0.9}OOM
  & \cellcolor[gray]{0.9}\underline{34.4} & \cellcolor[gray]{0.9}45.7 & \cellcolor[gray]{0.9}OOM & \cellcolor[gray]{0.9}OOM
  & \cellcolor[gray]{0.9}\textbf{5.1} & \cellcolor[gray]{0.9}5.9 & \cellcolor[gray]{0.9}6.6 & \cellcolor[gray]{0.9}8.0\\
& \cellcolor[gray]{0.9}\textbf{Time}
  & \cellcolor[gray]{0.9}--& \cellcolor[gray]{0.9}35.7 & \cellcolor[gray]{0.9}--& \cellcolor[gray]{0.9}--  & \cellcolor[gray]{0.9}157.1 & \cellcolor[gray]{0.9}\underline{35.5} & \cellcolor[gray]{0.9}--& \cellcolor[gray]{0.9}--  & \cellcolor[gray]{0.9}\textbf{6.6} & \cellcolor[gray]{0.9}11.4 & \cellcolor[gray]{0.9}16.4 & \cellcolor[gray]{0.9}26.4\\
& \textbf{GPU OOM}
  & \checkmark & \checkmark & \checkmark & \checkmark &  &  & \checkmark & \checkmark &  &  &  & \\
\midrule[1pt]
\multirow{5}{*}{\textbf{ML-20M}}
& \textbf{VRAM (A)}
  & -- & 19.7 & --& --& 11.5 & -- & -- & -- & 4.6 & 4.8 & 4.8 & 5.6\\
& \textbf{RAM (B)}
  & 12.6 & 1.3 & OOM & OOM & 4.5 & 6.6 & 9.6 & 12.8 & 1.2 & 1.2 & 1.2 & 1.2\\
& \cellcolor[gray]{0.9}\textbf{Total memory (A)+(B)}
  & \cellcolor[gray]{0.9}12.6 & \cellcolor[gray]{0.9}21.0 & \cellcolor[gray]{0.9}OOM & \cellcolor[gray]{0.9}OOM
  & \cellcolor[gray]{0.9}15.9 & \cellcolor[gray]{0.9}\underline{6.6} & \cellcolor[gray]{0.9}9.6 & \cellcolor[gray]{0.9}12.8
  & \cellcolor[gray]{0.9}\textbf{5.8} & \cellcolor[gray]{0.9}6.0 & \cellcolor[gray]{0.9}6.0 & \cellcolor[gray]{0.9}6.8\\
& \cellcolor[gray]{0.9}\textbf{Time}
  & \cellcolor[gray]{0.9}223.0 & \cellcolor[gray]{0.9}\textbf{59.8} & \cellcolor[gray]{0.9}--& \cellcolor[gray]{0.9}--  & \cellcolor[gray]{0.9}121.1 & \cellcolor[gray]{0.9}187.6 & \cellcolor[gray]{0.9}2735.5 & \cellcolor[gray]{0.9}7432.8
  & \cellcolor[gray]{0.9}\underline{102.2} & \cellcolor[gray]{0.9}156.6 & \cellcolor[gray]{0.9}209.1 & \cellcolor[gray]{0.9}314.1\\
& \textbf{GPU OOM}
  & \checkmark &  & \checkmark & \checkmark &  & \checkmark & \checkmark & \checkmark &  &  &  & \\
\bottomrule[1pt]
\end{tabular}
\vspace{-1mm}
\end{table*}
\noindent\textbf{Implementation details.} We adopt the optimal hyperparameters for competitors through extensive hyperparameter tuning on the validation set. For all experiments, we use default hyperparameters $\gamma$, $s$, and polynomial order $N$ for \textsf{Mem-GF}, unless otherwise stated: $(\gamma, s, N)$ = (0.55, 0.9, 5) for Yelp; (0.5, 0.9, 2) for Amazon-book; and (0.5, 1.1, 2) for ML-20M. We set $K = N + 1$ for all datasets. All training-based methods are optimized using the Adam optimizer \cite{kingma2015adam} with a batch size of 1024. Our experiments are conducted on a single server equipped with Intel (R) 12-Core (TM) i7-9700K CPUs @ 3.60 GHz, NVIDIA GeForce RTX A6000 GPUs with \textbf{48GB VRAM}, and \textbf{128GB of RAM}. 

All GF-based CF algorithms \cite{shen2021powerful,park2024turbo,steck2019embarrassingly,liu2023personalized,xia2024hierarchical} are implemented using PyTorch tensors \cite{paszke2019pytorch} to ensure a consistent experimental environment. Whenever a method exceeds the size of GPU VRAM, we revert to CPU-only computation leveraging system RAM to complete the process without disrupting the pipeline. For FPSR \cite{wei2023fine}, we use $\tau=0.3$ as a default hyperparameter to generate 6 partitions. For Turbo-CF, we identically use the three predefined filters presented in \cite{park2024turbo, kim2025leveraging}. Since the filtering matrix is constructed as $\bar{\mathbf R}^{\top}\bar{\mathbf R}$, using $\bar{\mathbf r}_u$ also aligns the user signal with the same normalization and adjustment applied to the item similarity matrix. This alignment is particularly helpful for the ML-20M dataset, where interactions are highly concentrated on popular items, because it mitigates popularity-dominated signal amplification.

To evaluate the scalability of our \textsf{Mem-GF} in Section~\ref{section 4.3}, we generated sets of synthetic interaction matrices with varying dimensions and sparsity levels. In particular, we employ the random matrix generation function provided by the \texttt{scipy.random} module in the SciPy package \cite{virtanen2020scipy} to create these matrices. By default, the number of users, items, and nonzero entries (\textit{i.e.}, the number of interactions) was set to 5,000, 5,000, and 225,000, respectively. In order to assess scalability across a wide range of problem sizes, the experiments are carried out according to the following setup:
\begin{itemize}[leftmargin=*]
    \item \textbf{Varying the number of users/items:} The number of users and items is systematically increased to each of the following values: 1,000, 5,000, 10,000, 15,000, 20,000, 30,000, 40,000, 50,000, 80,000, and 100,000.
    \item \textbf{Varying the number of interactions:} The number of nonzero entries (\textit{i.e.}, $\mathrm{nnz}(R)$) is scaled to 225,000, 1,125,000, 2,250,000, 4,500,000, 6,750,000, 9,000,000, and 11,250,000.
\end{itemize}
For each configuration, we record the corresponding memory usage and runtime of our \textsf{Mem-GF}. These systematic variations allow us to quantify how the performance of our method scales \textit{w.r.t.} the number of users, items, and interactions.

\subsection{Preprocessing Efficiency Analysis (RQ1)} 

Table \ref{resource_consumption} summarizes the comparison of \textsf{Mem-GF} with existing GF-based recommendation methods in terms of memory usage and runtime on the three benchmark datasets during preprocessing ({\it i.e.}, offline computation). Table \ref{tab:runtime} shows the runtime comparison among \textsf{Mem-GF} and three representative CF methods that require model training (NGCF, LightGCN, and DiffRec). Our key observations are as follows:

\begin{enumerate}[label=(\roman*)]
    \item Leveraging the low-dimensional Krylov subspace, \textsf{Mem-GF} significantly reduces both memory consumption and computation time compared to other GF-based CF competitors.
    \item On Amazon-book, where the set of items is relatively large, most GF-based CF competitors encounter OOM errors on the single GPU device. In contrast, \textsf{Mem-GF} efficiently operates within around 5GB of VRAM under the first-order polynomial filter, demonstrating up to a \textbf{574\% memory advantage} and a \textbf{438\% runtime speedup} compared to the second-best performers. Even with higher-order polynomial filters, \textsf{Mem-GF} maintains a VRAM usage of around 8GB.
    \item On ML-20M, the extremely large number of interactions (20M) triggers buffer OOM issues for many competing methods when performing sparse matrix multiplications on the GPU, forcing them to rely solely on CPU computations and resulting in prolonged runtime. On the contrary, \textsf{Mem-GF} seamlessly handles high-order filters without such memory bottlenecks.
    \item Table~\ref{tab:runtime} demonstrates the benefits of \textsf{Mem-GF} over representative CF methods that require model training in terms of runtime. Because \textsf{Mem-GF} is entirely training-free, it circumvents the costly iterative parameter optimization loops. Notably, on the datasets having a large amount of interactions, such as ML-20M, \textsf{Mem-GF} significantly outperforms two-tower models based on pairwise optimization ({\it e.g.}, NGCF and LightGCN), highlighting its practical value in scenarios where both accuracy and speed are critical.
\end{enumerate}
\vspace{-3mm}

\begin{table}[t]
\footnotesize
\centering
\caption{Runtime (in seconds) of \textsf{Mem-GF} and representative competitors that require model training. Here, runtime represents training time for NGCF, LightGCN, and DiffRec, while it refers to processing time for \textsf{Mem-GF}.}
\label{tab:runtime}
\vspace{-1mm}
\begin{tabular}{lccccc}
\hline
\textbf{} & \textbf{NGCF}   & \textbf{LightGCN}  & \textbf{DiffRec} & \textbf{\textsf{Mem-GF}} \\ \hline
Yelp & $1.5 \times 10^4$ & $1.1 \times 10^4$ &  $1.4 \times 10^4$ & $\bf 1.5$\\
Amazon-book &$7.1 \times 10^4$ & $4.6 \times 10^4$ &$3.7 \times 10^4$  & $\bf 6.6$\\
ML-20M & $4.6 \times 10^5$& $3.1 \times 10^5$ & $1.1 \times 10^4$ & $\bf 102.2$\\
Model training & \cmark & \cmark & \cmark & \xmark \\
\hline
\end{tabular}
\end{table}
\begin{table}[t!]
    \footnotesize
    \centering
    \caption{Inference time and peak GPU memory usage for representative GF-based methods and \textsf{Mem-GF} with different polynomial orders on the Yelp dataset.}
    \vskip -0.1in
    \begin{tabular}{lccc}
    \toprule
    \textbf{Method}    & \textbf{Order} & \textbf{Time (s)} & \textbf{VRAM (GB)} \\
    \midrule
    GF-CF              & -              & 0.1197            & 28.79             \\
    Turbo-CF           & 1              & 0.1348            & 6.60              \\
    \textsf{Mem-GF-1}  & 1              & 0.0044            & 0.57              \\
    \textsf{Mem-GF-2}  & 2              & 0.0077            & 0.57              \\
    \textsf{Mem-GF-3}  & 3              & 0.0123            & 0.57              \\
    \textsf{Mem-GF-5}  & 5              & 0.0145            & 0.57              \\
    \bottomrule
    \end{tabular}
    \label{tab:inference_comparison}
\end{table}

\subsection{Inference Efficiency Analysis (RQ2)} 
Table~\ref{tab:inference_comparison} reports the single-batch inference cost ({\it i.e.}, the convolution with graph signals) on the Yelp dataset, as other datasets exhibit similar trends. The comparison assumes that the graph filter ($H(\mathbf{L})$ in \eqref{graph conv}) has already been precomputed and focuses only on runtime and peak GPU memory usage during inference.

\begin{enumerate}[label=(\roman*)]
    \item By confining the graph convolution entirely within a low-dimensional Krylov subspace, \textsf{Mem-GF} completes inference in $0.0044$ seconds, yielding up to a 26.2$\times$ speedup over GF-CF and Turbo-CF, both of which require convolution with a full $|\mathcal{I}|\times|\mathcal{I}|$ graph filter. Even with the fifth-order filter, \textsf{Mem-GF} maintains low latency, supporting real-time recommendation systems.
    \item Since \textsf{Mem-GF} avoids explicit construction and storage of the full-size $|\mathcal{I}|\times|\mathcal{I}|$ matrix, it achieves over a 10.5$\times$ reduction in peak VRAM usage compared to Turbo-CF. The memory footprint remains nearly unaffected by the polynomial order because the compact Krylov basis dominates the filtering cost, whereas GF-CF requires dense filter storage for ideal LPF computations.
\end{enumerate}

\begin{figure}[t]
    \centering
    \includegraphics[width=0.99\columnwidth]{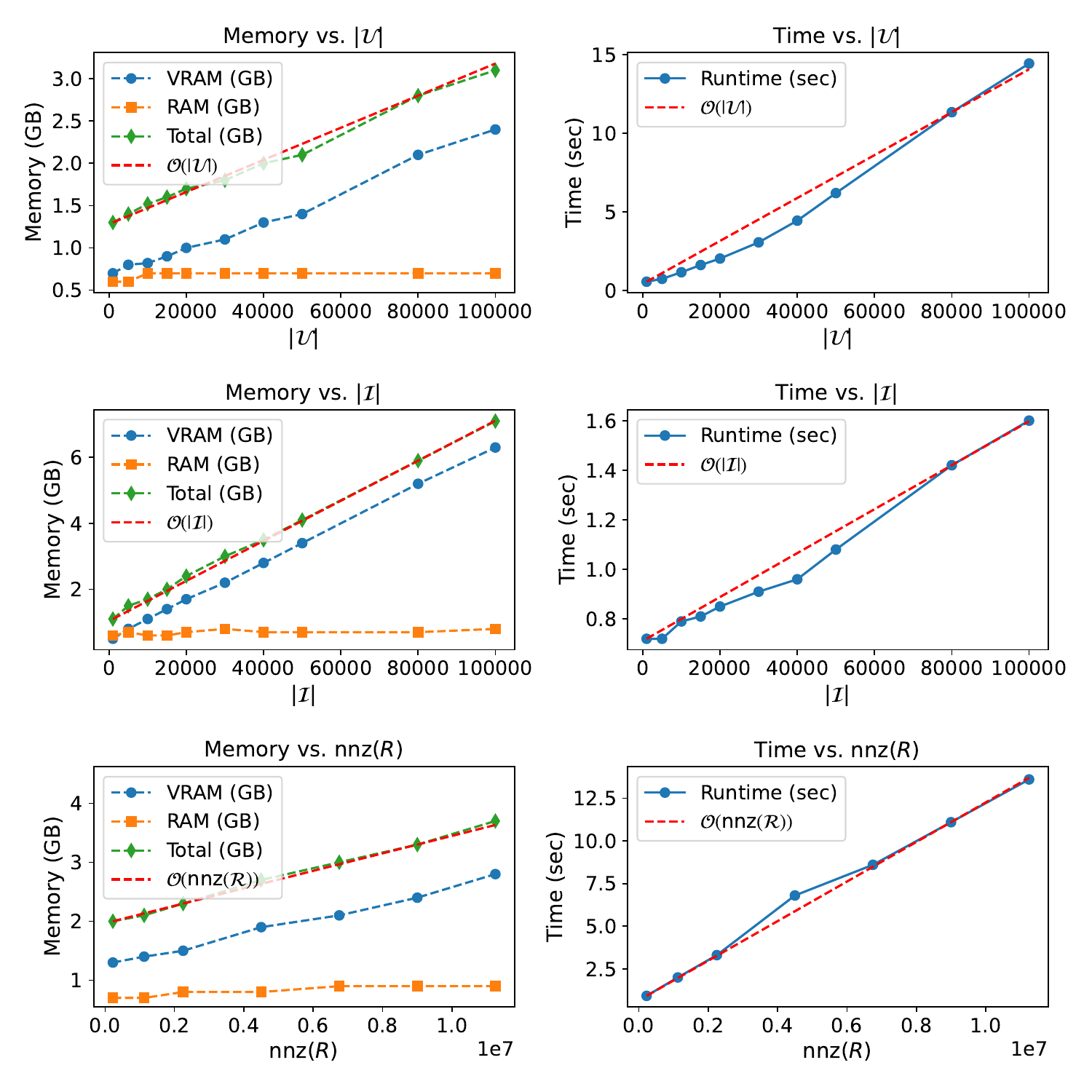}
    \vspace{-3mm}
   \caption{Memory usage (left) and runtime (right) of \textsf{Mem-GF} as the number of users ($|\mathcal{U}|$), items ($|\mathcal{I}|$), and non-zero interactions ($\mathrm{nnz}(R)$) increases. Here, the linear asymptotic line is depicted as a red-dashed line for reference.}
   \vspace{-3mm}
    \label{fig;complexity}
\end{figure}

\subsection{Scalability Analysis (RQ3)}
\label{section 4.3}

To empirically validate our theoretical complexity analysis in Section~\ref{sec 3.2.4}, we conduct experiments on synthetic datasets (\textit{i.e.}, interaction matrices) with varying scales of users, items, and interaction densities.  Specifically, in Fig. \ref{fig;complexity}, we present how \textsf{Mem-GF}'s memory usage and runtime for preprocessing scale as $|\mathcal{U}|$, $|\mathcal{I}|$, and $\mathrm{nnz}(R)$ grow. By default, the number of users, items, and interactions was set to 5,000, 5,000, and 225,000, respectively. Furthermore, in Fig. \ref{fig;complexity_others}, we compare \textsf{Mem-GF} with two representative GF-based CF competitors (GF-CF and Turbo-CF) when $|\mathcal{I}|$ increases. For a fair comparison with GF-CF, both Turbo-CF and \textsf{Mem-GF} employ the first-order polynomial graph filter. We highlight the following observations:

\begin{enumerate}[label=(\roman*)]
    \item As shown in the left side of Fig. \ref{fig;complexity}, \textsf{Mem-GF} exhibits memory usage growing nearly \textbf{linearly} with $|\mathcal{U}|$, $|\mathcal{I}|$, and $\mathrm{nnz}(R)$. Notably, the memory footprint is significantly reduced compared to the case of forming the item similarity matrix of size $|\mathcal{I}|\times|\mathcal{I}|$, allowing \textsf{Mem-GF} to make full use of GPU acceleration even on large datasets. We note that the size of RAM does not tend to scale with the increasing data dimension as long as VRAM is sufficiently available.
    
    \item As shown in the right side of Fig. \ref{fig;complexity}, the runtime of \textsf{Mem-GF} also scales \textbf{linearly} as each data dimension grows. This validates the scalability of \textsf{Mem-GF}. Moreover, since \textsf{Mem-GF} efficiently operates within the GPU VRAM, its runtime only takes less than 11 seconds even when $\mathrm{nnz}(R)$ increases up to $10^7$ via the exploitation of the GPU resource.

    \item In Fig. \ref{fig;complexity_others}, we observe that \textsf{Mem-GF} is substantially more scalable in both memory and runtime, compared to GF-CF and Turbo-CF. In particular, GF-CF incurs excessive memory overhead due to its dense ideal LPF matrix having dimensionality of $(|\mathcal{I}|,|\mathcal{I}|)$, thus yielding superlinear complexity in practice. 
\end{enumerate}
These results also indicate how sparsity affects \textsf{Mem-GF}: increasing $\mathrm{nnz}(\mathbf{R})$ raises runtime through sparse matrix-vector products, but the growth remains close to linear because the full item similarity matrix is never formed.
\begin{figure}[t]
    \centering
    \includegraphics[width=0.95\columnwidth]{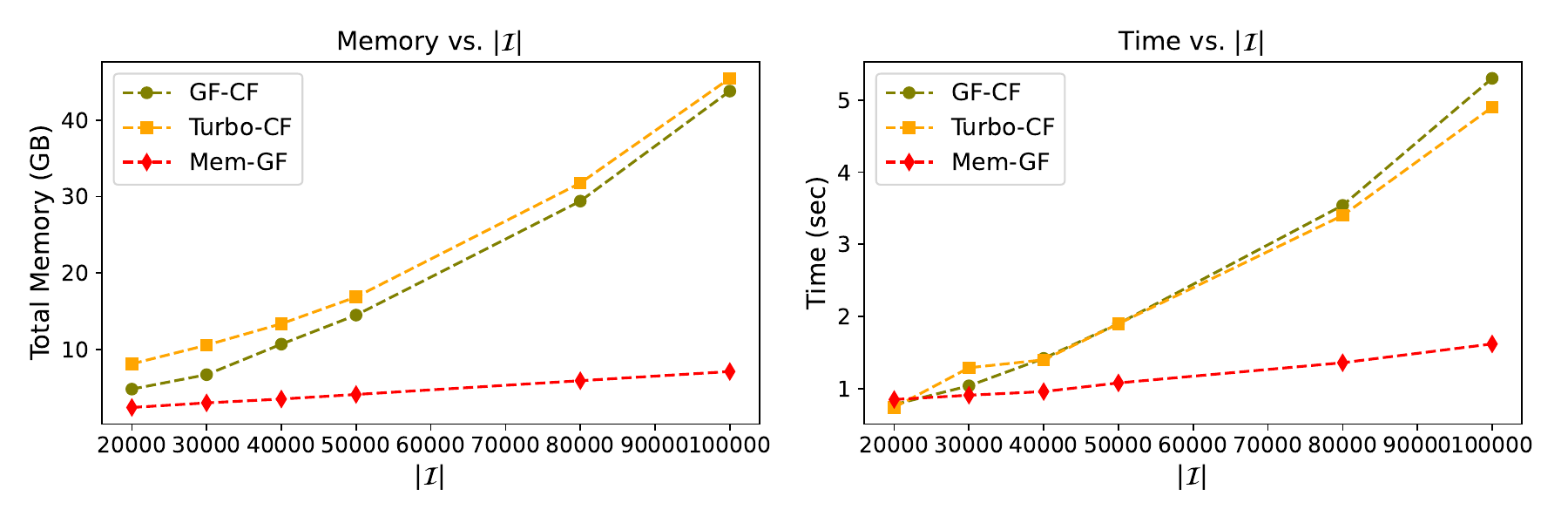}
    \vspace{-3mm}
   \caption{Total memory (left) and runtime (right) comparison among representative GF-based CF methods {\it w.r.t.} $|\mathcal{I}|$.}
    \label{fig;complexity_others}
\end{figure}
\begin{table}[!t]\centering
\setlength\tabcolsep{5.0pt}
\scriptsize
  \captionsetup{skip=2pt}
  \caption{Performance comparison among \textsf{Mem-GF} and competitors. The best and second-best performers are highlighted in bold and underlined, respectively. The improvements of \textsf{Mem-GF} over the best competitors are all statistically significant with $p$-values $\leq 0.01$.}
  \label{tab:acc results}
  \begin{tabular}{c|cc|cc|cc}
    \toprule[1pt]
    \multicolumn{1}{c|}{}&\multicolumn{2}{c|}{Yelp}&\multicolumn{2}{c|}{Amazon-book}&\multicolumn{2}{c}{ML-20M}\\
    \cmidrule{1-7}
           Method &   Recall& NDCG&  Recall& NDCG& Recall & NDCG\\
    \midrule[1pt]
    MF-BPR  &   0.0433 & 0.0354 &  0.0250  & 0.0190 &   0.0219   & 0.1108    \\
    NeuMF   &   0.0451 & 0.0363 &  0.0258  & 0.0200 &    0.0288  &    0.1437\\
    NGCF    &   0.0579 & 0.0477 &   0.0344 & 0.0263 &   0.0307   &   0.1409  \\
    LightGCN&   0.0649 & 0.0530 &   0.0411 & 0.0315 &    0.0402  &  0.1702   \\
    DGCF   & 0.0654  & 0.0534 & 0.0422  &0.0324  &   OOM   &  OOM    \\
    SGL  &  0.0675 & 0.0555&  0.0478 & 0.0379&    0.0388&       0.1677    \\
    SGDE &  0.0693 & 0.0563&  0.0483 & 0.0391&    0.0398&       0.1683   \\
    Multi-DAE &   0.0562 & 0.0434 &   0.0410 & 0.0322 &    0.0385  & 0.1666    \\
    Multi-VAE &   0.0584 & 0.0450 &   0.0407 & 0.0315 &   0.0402   & 0.1702    \\
    RecVAE &   0.0579 & 0.0432 &   0.0400 & 0.0302 &   0.0414   & 0.1732    \\
    SVD-AE  &  0.0681 & 0.0572 & 0.0602 & 0.0479 & 0.0536  & 0.2379    \\
    SimpleX  &   0.0701&  0.0575&  0.0583  & 0.0400 &  0.0304    &    0.1451 \\ 
    LinkProp  &   0.0676 &  0.0559 &  0.0684  & 0.0559 &     OOM &   OOM  \\ 
    LinkProp-Multi   &   0.0690 &  0.0571  &  0.0721  & 0.0588 &    OOM  &  OOM   \\ 
    DiffRec   &  0.0656 & 0.0552 &  0.0514&0.0418  & \underline{0.0549}     & 0.2258    \\
    L-DiffRec  &  0.0614 & 0.0518 &  0.0502 & 0.0409 &  0.0470   &0.1837    \\
    FlowCF  &  0.0686 & 0.0587 &  0.0497 & 0.0431 &  0.0531   &\underline{0.2340 }   \\
    \cmidrule{1-7}
    EASE  & 0.0657  &0.0552 & 0.0710 &0.0567  &   0.0522  & 0.2264  \\
    GF-CF   &   0.0697 & 0.0571 &  0.0710  & 0.0584 &  0.0487   & 0.2076  \\    
    PGSP  & \textbf{0.0710}  & 0.0583& 0.0712 & 0.0587 & OOM   & OOM   \\
    FPSR &  0.0703 & \underline{0.0584} & 0.0715 & 0.0588 &0.0493  &  0.2104  \\
    Turbo-CF & 0.0693 & 0.0574 & \underline{0.0730} & \underline{0.0611} & 0.0534     &0.2299   \\
    HiGSP  &  OOM & OOM& OOM & OOM & OOM  &  OOM  \\
    \rowcolor{gray!20}\textsf{\textbf{Mem-GF (Ours)}}  & \underline{0.0705}  &\textbf{ 0.0586}& \textbf{0.0753} &\textbf{0.0620}  &   \textbf{ 0.0566}  & \textbf{0.2402}    \\
    \bottomrule[1pt]
  \end{tabular}
  \vspace{-1mm}
\end{table}

\subsection{Recommendation Accuracy (RQ4)}
Table~\ref{tab:acc results} summarizes the recommendation accuracy of \textsf{Mem-GF} and various competitors across three distinct datasets. The main findings are as follows:
\begin{enumerate}[label=(\roman*)]
    \item Despite its exceptional computational efficiency, \textsf{Mem-GF} achieves the highest accuracy for all datasets (except the recall on Yelp), underscoring its robustness across diverse data characteristics.
    \item In particular, the gains of \textsf{Mem-GF} over training-dependent models are indeed substantial across all datasets.
    \item \textsf{Mem-GF} consistently exhibits gains over Turbo-CF, another polynomial GF approach. On the Amazon-book dataset, Turbo-CF was restricted only to the first-order linear filter to avoid OOM issues (see Table \ref{resource_consumption}), whereas \textsf{Mem-GF} can freely use higher-order filters without incurring any resource scarcity issues on a single device. This design can improve accuracy empirically, not because memory reduction itself guarantees better predictions, but because expressive polynomial filters can be selected within the user-specific Krylov subspace without being constrained by the storage cost of full-size graph filters.
\end{enumerate}
\vspace{-3mm}
\begin{table}[t]
\centering

\caption{Recall performance across different values of $K$ and polynomial filters. Here, \textsf{Mem-GF-k} denotes an ablation model that omits the Krylov basis construction in \textsf{Mem-GF}. Cases where $K \leq N$ are marked in gray.}
\label{tab:theorem_exp}
\scriptsize
\begin{tabular}{c|cccc}
\toprule
\multicolumn{5}{c}{\textbf{Yelp}} \\
\midrule
$K$ & \textsf{Mem-GF-1} & \textsf{Mem-GF-2} & \textsf{Mem-GF-3} & \textsf{Mem-GF-5} \\
\midrule
1 & \cellcolor{gray!20}0.0035 & \cellcolor{gray!20}0.0035 & \cellcolor{gray!20}0.0035 & \cellcolor{gray!20}0.0035 \\
2 & 0.0684 & \cellcolor{gray!20}0.0684 & \cellcolor{gray!20}0.0684 & \cellcolor{gray!20}0.0684 \\
3 & 0.0684 & 0.0692 & \cellcolor{gray!20}0.0683 & \cellcolor{gray!20}0.0695 \\
4 & 0.0684 & 0.0692 & 0.0683 & \cellcolor{gray!20}0.0695 \\
5 & 0.0684 & 0.0692 & 0.0683 & \cellcolor{gray!20}0.0701 \\
6 & 0.0684 & 0.0692 & 0.0683 & 0.0703 \\
\midrule
\textsf{Mem-GF-k} & 0.0684 & 0.0692 & 0.0683 & 0.0703 \\
\midrule[1pt]
\multicolumn{5}{c}{\textbf{Amazon-book}} \\
\midrule
$K$ & \textsf{Mem-GF-1} & \textsf{Mem-GF-2} & \textsf{Mem-GF-3} & \textsf{Mem-GF-5} \\
\midrule
1 & \cellcolor{gray!20}0.0011 & \cellcolor{gray!20}0.0011 & \cellcolor{gray!20}0.0011 &\cellcolor{gray!20} 0.0011 \\
2 & 0.0710 & \cellcolor{gray!20}0.0628 & \cellcolor{gray!20}0.0710 & \cellcolor{gray!20}0.0699 \\
3 & 0.0710 & 0.0754 &\cellcolor{gray!20} 0.0707 & \cellcolor{gray!20}0.0710 \\
4 & 0.0710 & 0.0754 & 0.0707 & \cellcolor{gray!20}0.0745 \\
5 & 0.0710 & 0.0754 & 0.0707 & \cellcolor{gray!20}0.0739 \\
6 & 0.0710 & 0.0754 & 0.0707 & 0.0739 \\
\midrule
\textsf{Mem-GF-k}& 0.0710 & 0.0754 & 0.0707 & 0.0739 \\
\midrule[1pt]
\multicolumn{5}{c}{\textbf{ML-20M}} \\
\midrule
$K$ & \textsf{Mem-GF-1} & \textsf{Mem-GF-2} & \textsf{Mem-GF-3} & \textsf{Mem-GF-5} \\
\midrule
1 & \cellcolor{gray!20}0.0248 & \cellcolor{gray!20}0.0248 & \cellcolor{gray!20}0.0248 & \cellcolor{gray!20}0.0248 \\
2 & 0.0554 & \cellcolor{gray!20}0.0554 &\cellcolor{gray!20} 0.0554 &\cellcolor{gray!20} 0.0554 \\
3 & 0.0554 & 0.0566 &\cellcolor{gray!20} 0.0552 & \cellcolor{gray!20}0.0566 \\
4 & 0.0554 & 0.0566 & 0.0516 & \cellcolor{gray!20}0.0565 \\
5 & 0.0554 & 0.0566 & 0.0516 & \cellcolor{gray!20}0.0564 \\
6 & 0.0554 & 0.0566 & 0.0516 & 0.0564 \\
\midrule
\textsf{Mem-GF-k}& 0.0554 & 0.0566 & 0.0516 & 0.0564 \\
\bottomrule
\end{tabular}
\end{table}
\subsection{Empirical Validation of Theorem \ref{thm:polynomial_approx} (RQ5)}
\label{app:exp_thm}

Table~\ref{tab:theorem_exp} summarizes our experiments across all datasets, where we varied the number of Krylov basis vectors ($K$) for polynomial filters of different maximum polynomial orders ($N$). Here, \textsf{Mem-GF-k} denotes an ablation model that omits the Krylov basis construction in \textsf{Mem-GF}. It is observed that, for an $N$-th order polynomial filter, performance remains consistent with that of \textsf{Mem-GF-k} when $K \geq N+1$. This finding confirms that setting $K = N+1$ produces no approximation error under the theorem condition. In contrast, the gray entries with $K \leq N$ illustrate regimes where the guarantee does not apply and performance can become unstable.

\subsection{Analysis of Polynomial Graph Filters (RQ6)}
Table~\ref{filter_table} shows that the optimal polynomial filter varies depending on datasets; for instance, the fifth-order filter (\textsf{Mem-GF-5}) achieves the highest recall on Yelp, whereas the second-order filter (\textsf{Mem-GF-2}) performs best on Amazon-book and ML-20M. As expected, higher-order filters incur longer runtime. Nonetheless, since the overall computation time remains manageable and the search space is restricted to four predefined filters, it is feasible to perform a validation set-based search. Increasing the polynomial order $N$ can improve filter expressiveness, but a sufficiently large Krylov subspace size $K$ is required; in particular, the exactness guarantee holds under exact arithmetic when $N<K$, whereas regimes with $K\le N$ may become unstable.

\begin{table}[t!]
    \scriptsize
    \centering
    \caption{Runtime for preprocessing and recall according to four different types of polynomial graph filters in \textsf{Mem-GF}.}
    \vskip -0.1in
    \begin{tabular}{lcccccc}
    \toprule
    & \multicolumn{2}{c}{\textbf{Yelp}} & \multicolumn{2}{c}{\textbf{Amazon-book}} & \multicolumn{2}{c}{\textbf{ML-20M}}  \\
    \cmidrule(r){2-3} \cmidrule(r){4-5} \cmidrule(r){6-7}
    & Runtime & Recall & Runtime & Recall & Runtime & Recall \\
    \midrule
    \textsf{Mem-GF-1} & \textbf{1.5s} &  0.0687& \textbf{6.6s} &  0.0730   & \textbf{102.2s}  &0.0554 \\
    \textsf{Mem-GF-2} & 2.9s & 0.0692 & 11.4s &  \textbf{0.0753 } & 156.6s & \textbf{0.0566}\\
    \textsf{Mem-GF-3} & 4.3s & 0.0683 & 16.4s &  0.0693  & 209.1s & 0.0516\\
    \textsf{Mem-GF-5} & 6.5s & \textbf{0.0705} &  26.4s & 0.0739  & 314.1s & 0.0564\\
    \bottomrule
    \end{tabular}
    \label{filter_table}
\end{table}
\begin{figure}[t]
    \centering
    \includegraphics[width=0.95\columnwidth]{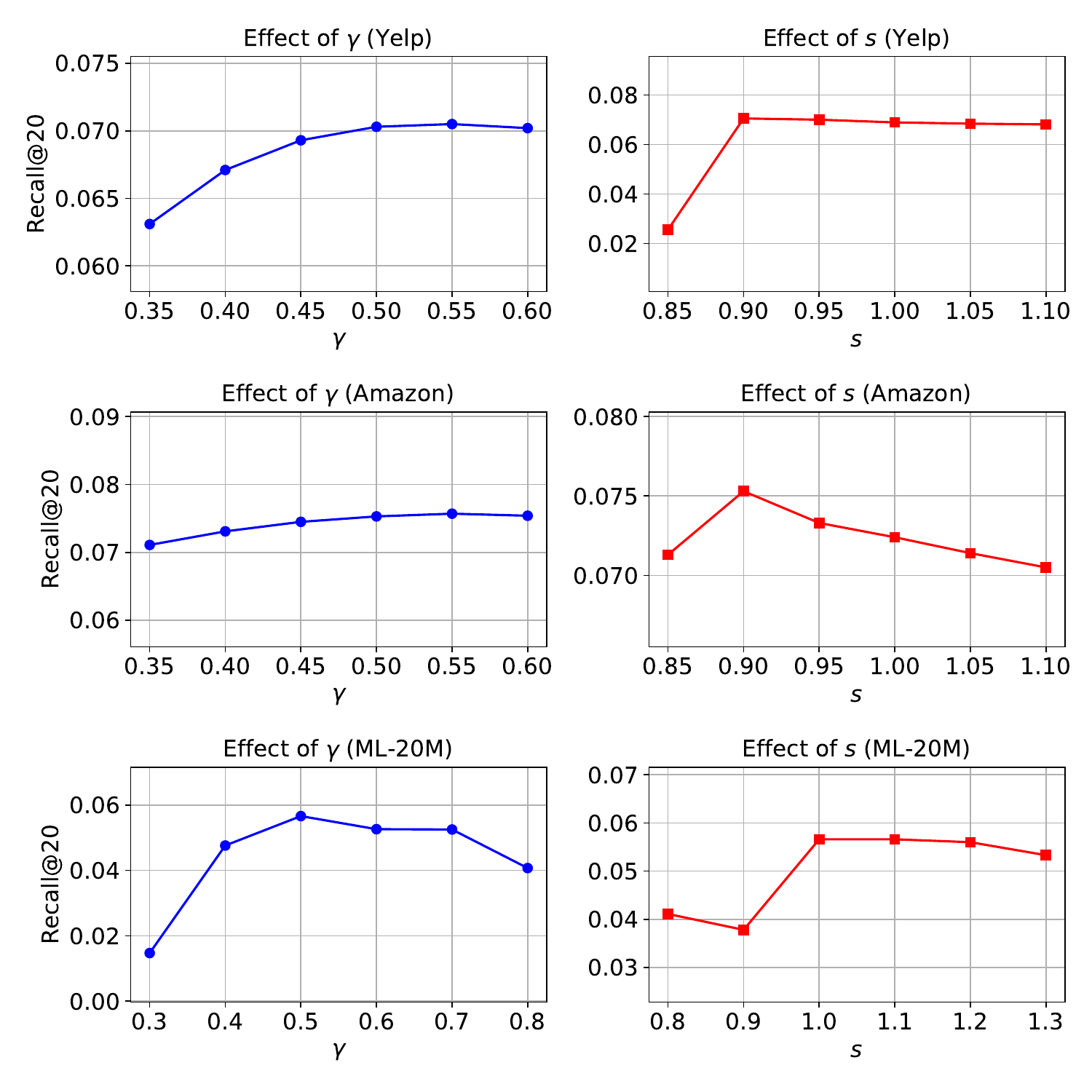}
       \caption{The effect of two hyperparameters $\gamma$ and $s$, on the accuracy of \textsf{Mem-GF}.}
    \label{fig:sa_plot}
\end{figure}

\subsection{Sensitivity Analysis (RQ7)}
\label{app:sa}
Fig. \ref{fig:sa_plot} visualizes the impact of two key hyperparameters in \textsf{Mem-GF}: the normalization parameter $\gamma$ in \eqref{conventional_graph} and the adjustment parameter $s$ in \eqref{eq:barR_def}. First, while Yelp attains its highest accuracy at $\gamma = 0.55$, both Amazon-book and ML-20M achieve optimal performance at $\gamma = 0.5$. Notably, using the conventional choice of $\gamma = 0.5$ still yields sufficiently robust accuracy across all datasets. Second, the optimal values of $s$ are $0.9$ on both Yelp and Amazon-book, and $1.0$ on ML-20M, indicating that setting closely to $s=1$ generally yields robust performance. In addition, we observe a notable performance drop on Yelp and ML-20M when $s<0.9$ and $s < 1.0$, respectively, whereas, outside these ranges, performance remains relatively stable. This aligns with the findings in \cite{park2024turbo}, highlighting the importance of selecting proper ranges of $s$ to ensure optimal recommendation accuracy.

\subsection{Practical Deployment and Limitations}
\textsf{Mem-GF} can be integrated into real-world recommender systems as a training-free GF-based CF module for candidate generation or scoring, and its output scores can be passed to subsequent ranking or re-ranking modules. Its user-specific Krylov subspace construction and score computation are naturally partitionable across workers or GPUs, since they can be performed independently for different users once the adjusted interaction matrix is available. Fully distributed system-level optimization, including distributed data management and cache-refresh strategies, is left as future work. The main limitations are as follows. First, \textsf{Mem-GF} is most directly suited to item similarity graphs built from user--item interactions; social networks, knowledge graphs, or other heterogeneous graph structures may require additional graph filter designs. Second, extremely sparse or noisy user signals, including cold-start-like cases with very few observed interactions, may degrade recommendation quality because the resulting user-specific Krylov subspace may contain insufficient or noisy preference information. In addition, very high polynomial orders may increase numerical sensitivity or computational overhead when $K$ is not sufficiently large. These issues can be mitigated through validation-based tuning of $K$, $N$, $\gamma$, and $s$.

\section{Conclusions and Outlook}
\label{section 5}
In this paper, we introduced \textsf{Mem-GF}, a training-free and memory-efficient GF method for CF. By operating within user-specific Krylov subspaces, \textsf{Mem-GF} eliminates the need to construct or store the full item similarity matrix, achieving substantial reductions in memory usage and preprocessing/inference time. In addition to scalability, \textsf{Mem-GF} supports expressive polynomial filters with low overhead. We further provided the condition under which Krylov filtering incurs no approximation error. Extensive experiments across three benchmark datasets demonstrated that \textsf{Mem-GF} achieves up to 574\% memory reduction, 438\% preprocessing speedup, 2620\% inference speedup, and state-of-the-art recommendation accuracy. Future work includes extending \textsf{Mem-GF} to heterogeneous graph structures and distributed recommendation systems.

\section*{Acknowledgement}
This work was supported by the National Research Foundation of Korea (NRF), a South Korea grant funded by the Korea government (MSIT) (RS-2021-NR059723), and by SMEs Technology Innovation Development Program through the Technology Innovation and Promotion Agency (TIPA), funded by the Ministry of SMEs and Startups (RS-2024-00511332).
\bibliographystyle{IEEEtran}
\bibliography{citation}

\end{document}